\documentclass{article}

\usepackage{microtype}
\usepackage{subfigure}
\usepackage{tikz}
\usetikzlibrary{arrows.meta,shapes,positioning,calc}
\usepackage{amsmath, amssymb}
\usetikzlibrary{calc, positioning, arrows.meta, shapes.geometric, shadows, backgrounds, fit, decorations.pathreplacing, calligraphy}
\usepackage{booktabs} 
\usepackage{hyperref}

\usepackage[final]{cpal_2025}
\usepackage{amsmath}
\usepackage{amssymb}
\usepackage{mathtools}
\usepackage{amsthm}
\usepackage{multirow}
\usepackage[textwidth=0.5in]{todonotes}

\newtheorem{theorem}{Theorem}
\newtheorem{lemma}{Lemma}
\newtheorem{proposition}{Proposition}

\usepackage{amsthm}

\newtheorem{definition}{Definition}

\theoremstyle{remark}
\newtheorem*{remark}{Remark}

\usepackage[capitalize,noabbrev]{cleveref}

\definecolor{ForestGreen}{RGB}{5,166,88}
\definecolor{LavaRed}{RGB}{222,48,28}
\definecolor{LightGrey}{RGB}{180,180,180}

\author{
Ru Wang$^{1}$, \quad
Wei Huang$^{2,3}$, \quad
Selena Song$^{1}$, \quad
Haoyu Zhang$^{1}$ \\
Qian Niu$^{1}$, \quad
Yusuke Iwasawa$^{1}$, \quad
Yutaka Matsuo$^{1}$, \quad
Jiaxian Guo$^{4}$ \\
$^{1}$The University of Tokyo \quad
$^{2}$ RIKEN Center for Advanced Intelligence Project \\
$^{3}$ The Institute of Statistical Mathematics \quad
$^{4}$ Google Research Australia\\
}

\begin{document}

\title{Beyond In-Distribution Success: Scaling Curves of CoT Granularity for Language Model Generalization}

\maketitle

\begin{abstract}
Generalization to novel compound tasks under distribution shift is important for deploying transformer-based language models (LMs). This work investigates Chain-of-Thought (CoT) reasoning as a means to enhance OOD generalization. Through controlled experiments across several compound tasks, we reveal three key insights: (1) While QA-trained models achieve near-perfect in-distribution accuracy, their OOD performance degrades catastrophically, even with 10000k+ training examples; (2) the granularity of CoT data strongly correlates with generalization performance; finer-grained CoT data leads to better generalization; (3) CoT exhibits remarkable sample efficiency, matching QA performance with much less (even 80\%) data.
Theoretically, we demonstrate that CoT forces internalization of valid dependency structures, and thus can achieve better generalization. Further, we show that transformer positional embeddings can amplify generalization by emphasizing subtask condition recurrence in long CoT sequences. Our combined theoretical and empirical analysis provides compelling evidence for CoT reasoning as a crucial training paradigm for enabling LM generalization on multi-step reasoning tasks under structural distributional shifts..
\end{abstract}

\section{Introduction}
\vspace{-0.7em}
Transformer-based language models (LMs) \cite{brown2020language,chowdhery2022palm,touvron2023llama} have demonstrated unprecedented capabilities in knowledge retrieval \cite{kojima2022large,wei2022chain,wei2022emergent}  and reasoning \cite{hendrycks2020measuring,cobbe2021training}, driven by large-scale pretraining on diverse text corpora \cite{ouyang2022training,wei2021finetuned,longpre2023flan}. While these models excel at tasks with clear input-output mappings, their ability to generalize to compound tasks—those requiring dynamic planning, multi-step reasoning, and adaptation to evolving contexts—remains a critical challenge. Real-world applications such as GUI automation \cite{wang2024gui,yang2024aria, shen2024falcon, qinghong2024showui, verma2024adaptagent}, textual puzzle solving \cite{giadikiaroglou2024puzzle,saha2024language}, and strategic gameplay like poker \cite{guo2023suspicion,zhang2024agent,huang2024pokergpt} demand not only procedural knowledge but also the capacity to handle distribution shift between training and deployment environments. 

\begin{figure}[!tbh]
    \centering
        \includegraphics[width=0.9\linewidth]{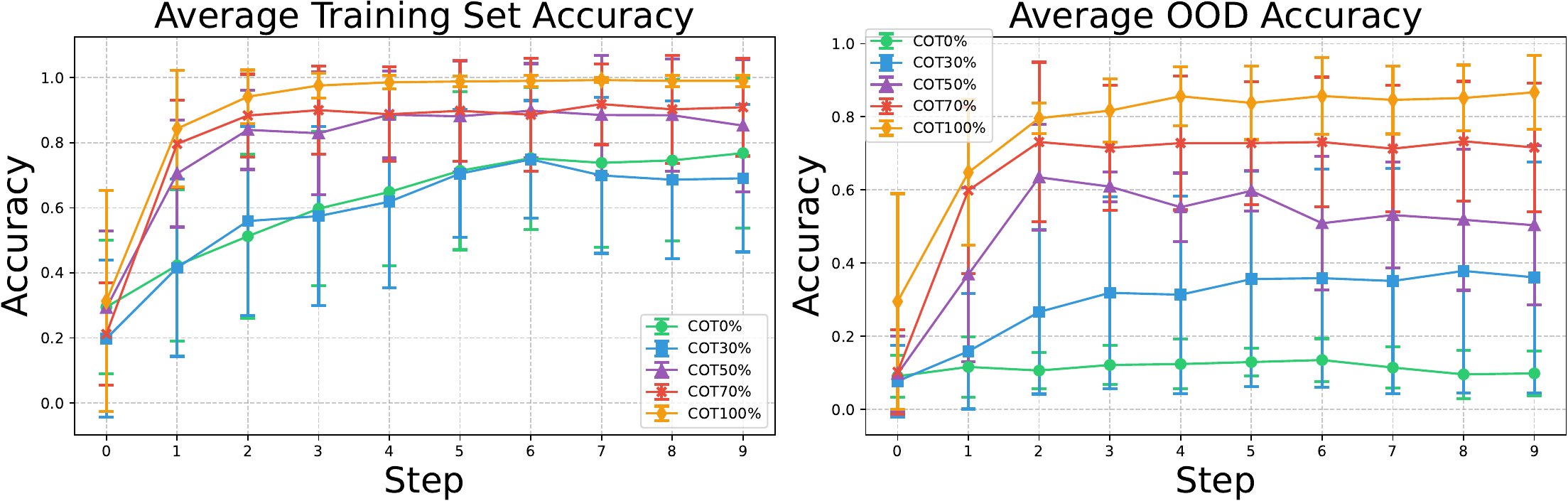}
    \vspace{-0.7\baselineskip}
    \caption{The illustration of the impact of the granularity of Chain-of-Thought on In-Distribution (IID) and Out-of-Distribution (OOD) performance. Left: IID performance. Right: OOD performance. Results are averaged over four compound tasks.  While models trained without CoT achieve high IID accuracy (~80\%), they exhibit a substantially poor generalization performance (~10\%) on OOD data.
}
    \label{fig:average}
\end{figure}

A central challenge lies in the misalignment between training data and deployment environments. For instance, consider a medical diagnosis system trained on historical patient records: it may falter when encountering novel symptom combinations during a disease outbreak or misalign with updated clinical guidelines (e.g., revised thresholds for "high-risk" biomarkers \cite{welsh2016prediction}). Such distribution shifts can significantly degrade the performance of LMs trained on limited training datasets, underscoring the importance of generalization for reliable deployment.

In this paper, we investigate how data collection impacts the generalization of LMs in compound tasks. Traditional approaches that prioritize direct instruction-result pairs (analogous to question-answering frameworks) optimize for task-specific accuracy but inadequately prepare models for compositional reasoning in unseen contexts. Recent advances in Chain-of-Thought (CoT) prompting \cite{wei2022chain}, which elicit intermediate reasoning steps, suggest that explicit process supervision can enhance generalization. However, acquiring high-quality CoT annotations at scale remains prohibitively expensive \cite{lightman2023let,kim2023cot}, raising a pivotal question: \textbf{How do different data collection strategies, \emph{e.g.}, result-oriented Q-A pairs and CoT—affect the generalization of LMs to novel compound tasks?}

To address this, we conduct controlled experiments using synthetic compound tasks that systematically introduce distribution shifts between training and evaluation environments. We evaluate LMs trained on two data paradigms: (1) result-oriented Q-A pairs and (2) CoT sequences of varying granularity. Our analysis reveals three critical insights(see Figure \ref{fig:average}):  (i) \textbf{Generalization Gap:} While both Q-A and CoT-trained LMs achieve near-perfect in-distribution accuracy, Q-A models exhibit severe performance degradation under distribution shifts—even with 10000k training examples. (ii) \textbf{Granularity-Generalization Tradeoff:} The granularity of CoT data strongly correlates with generalization performance; finer-grained CoT data leads to better generalization. (iii) \textbf{Sample Efficiency:} LMs trained with fine-grained CoT data demonstrate strong sample efficiency, achieving comparable generalization to Q-A pair training with substantially less data. These results suggest that even small amounts of high-quality CoT data can significantly enhance LM generalization, despite the challenges associated with its collection.

Inspired by the recent work \cite{liu2022transformers}, we further theoretically demonstrate that CoT training forces models to internalize valid reasoning paths. Leveraging transformer positional embeddings, we further show that explicitly emphasizing subtask conditions within CoT sequences further reduces shortcut reliance, especially in the long CoT setting.

 In summary, our paper makes the following key contributions: (1) We establish a controlled experimental framework to quantify the impact of data collection strategies on LM generalization under distribution shifts. Based on it, we demonstrate that LMs trained directly on question-answer pairs can achieve high accuracy on in-distribution data for new tasks but exhibit poor generalization, even when trained on substantial datasets (e.g., 10000k examples). (2) Through systematic scaling experiments, we demonstrate that fine-grained CoT data enhances generalization and sample efficiency, even with limited training data. (3) We provide theoretical insights into how CoT reasoning helps capture the dynamic state transitions of compound tasks, thereby improving generalization. Furthermore, we demonstrate that longer CoT chains, by repeatedly conditioning on sub-tasks, can further enhance the generalization capabilities of language models. We believe these findings offer a reliable guide for data collection practices when leveraging LMs for novel tasks.

\section{Related Work}
\vspace{-0.7em}
\paragraph{Generalization in LLMs}
Transformer-based language models \cite{vaswani2017attention} demonstrate strong performance on in-distribution tasks \cite{minaee2024largelanguagemodelssurvey,naveed2024comprehensiveoverviewlargelanguage,Xu_2024} but often struggle with OOD tasks due to challenges such as distribution shifts \cite{anil2022exploringlengthgeneralizationlarge,zhang2022delving}, shortcut learning \cite{liu2022transformers,geirhos2020shortcut}, and overfitting \cite{li2023transformers}, where the correlations they exploit no longer hold \cite{qian2022limitationslanguagemodelsarithmetic,nogueira2021investigatinglimitationstransformerssimple}.

Various approaches have been proposed to mitigate shortcut learning, including data augmentation \cite{hendrycks2020pretrainedtransformersimproveoutofdistribution,zhang2023unveilingtransformerslegosynthetic} and adversarial training \cite{jiang2019avoiding,taori2020when}, though many methods remain task-specific and lack generalization. Prior work \cite{chen2024alphamathzeroprocesssupervision} evaluates OOD performance by comparing models across datasets, but without precise control over distribution shifts. Our work introduces a fully controlled experimental setup, explicitly defining and manipulating shifts to better analyze model adaptation.
\paragraph{Chain-of-Thought Reasoning}
CoT reasoning enables multi-step derivations, allowing models to articulate reasoning \cite{wei2022chain,kojima2022large,zhu2024deductive,fu2022complexity,kim2024transformers}, though its mechanism remains unclear. Recent theoretical works apply circuit complexity theory to analyze CoT’s capabilities \cite{li2024chainthoughtempowerstransformers,abbe2024fartransformersreasonglobality,feng2023revealingmysterychainthought}, while other studies examine how CoT structures intermediate steps to decompose tasks \cite{ton2024understandingchainofthoughtllmsinformation,kudo2024thinktotalktalktothinkllmscome,yu2025llmsreallythinkstepbystep,yang2025chainofthought}.

CoT improves performance on tasks with shared reasoning patterns, even under distribution shifts \cite{li2024how,hu2024unveilingstatisticalfoundationschainofthought}. In zero-shot and few-shot settings, it leverages exemplars to apply pre-trained knowledge beyond training data \cite{kim2023cot}. Recent studies emphasize fine-grained CoT, where detailed intermediate steps enhance task learning and reduce shortcut reliance \cite{nguyen2023cof,chu-etal-2025-towards}.
\paragraph{Scaling Laws}
Scaling laws define the relationship between model size, data scale, and performance in LLMs \cite{kaplan2020scalinglawsneurallanguage,hoffmann2022trainingcomputeoptimallargelanguage,chowdhery2022palm}. However, scaling alone is insufficient under significant distribution shifts, as larger models may amplify shortcut learning or overfit to surface patterns \cite{micelibarone2022distributionallyrobustrecurrentdecoders}.
Optimizing the ratio of CoT reasoning in training data and validating its real-world effectiveness remain open challenges. Structured intermediate representations further aid in overcoming global obstacles and improving generalization \cite{abbe2024fartransformersreasonglobality}. Our findings indicate that incorporating CoT reasoning into the training process enhances the robustness of LLMs when addressing distribution shifts, particularly in out-of-distribution tasks.

\vspace{-0.7\baselineskip}
\section{Preliminary}
\vspace{-0.7em}
This paper investigates the generalization capabilities of transformer-based language models when applied to compound tasks characterized by dynamic dependencies among sub-tasks. Unlike standard sequential processing, compound tasks involve inter-related actions where the state of each action depends on a varying subset of previous actions. These dependencies are not static but evolve based on the task's progress, posing a unique challenge for models trained on sequential text. For example, consider preparing a complex meal with multiple dishes. The cooking state of pasta, for instance, depends on whether the sauce is freshly made or pre-made, demonstrating how dynamic dependencies can arise. This scenario reflects the intricate structure of many real-world tasks where a simple sequential reading fails to capture the underlying relationships between actions.

We formalize compound tasks through a state transition framework with dynamic dependencies, motivated by real-world scenarios like software development: subtask states (e.g., database setup, API integration) evolve through context-dependent interactions.
\subsection{Compound Task Structure}  
\vspace{-0.7em}
A compound task organizes complex operations through a hierarchical tree of subtasks. The structure comprises atomic actions at leaf nodes and interconnected subtasks at higher levels, with dynamic dependencies governing their execution flow. As illustrated in Figure \ref{fig:disbution_shift}, the system processes these tasks sequentially, evaluating dependencies and aggregating results from completed subtasks to achieve the overall objective. This flexible architecture allows for modification and reorganization of subtasks to accommodate evolving requirements.

\begin{figure}[t]
    \centering
        \includegraphics[width=0.9\linewidth]{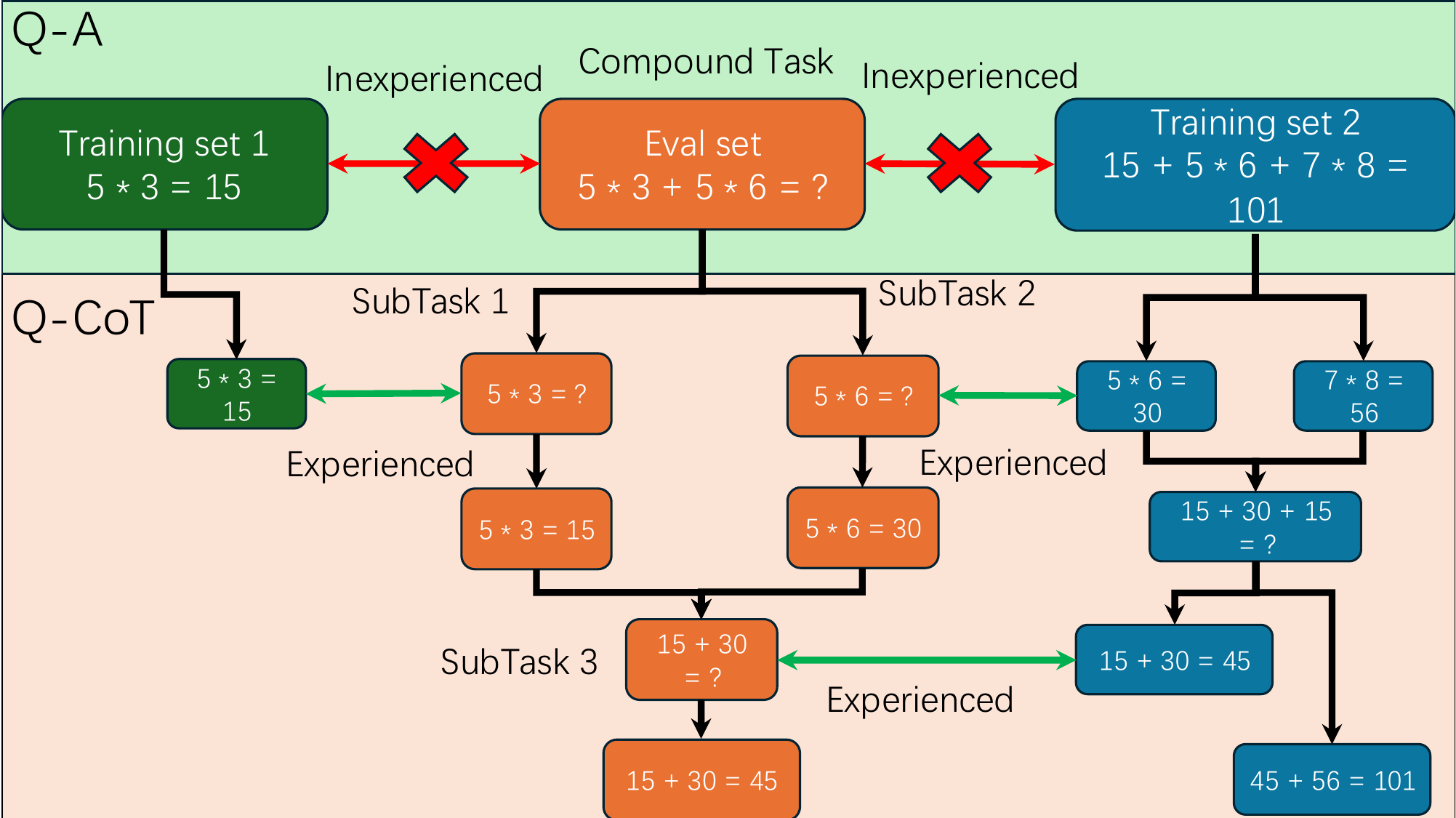}
    \vspace{-0.7\baselineskip}
    \caption{ Chain-of-thought alleviates distribution shift by breaking down complex problems into simpler, familiar sub-problems.}
    \label{fig:disbution_shift}
\end{figure}

To accommodate transformer models' input format requirements, we define the compound task at the token level. Figure \ref{fig:cp2} in Appendix provides a step-by-step illustration demonstrating how this definition integrates into the task structure.
\begin{definition}
\label{def:CP}
Given the input sequence $S = (s_1, \ldots, s_n)$ and subtask state sequence $Q = (q_1, \ldots, q_n)$ where $Q^i = (q_1, \ldots, q_i)$ denotes the prefix subsequence contains the first $i$ states of $Q$, CP($n$) is defined:
\setlength{\itemsep}{-1pt}
\item \textbf{Dynamic Graph Evolution:} At step $i+1$, the dependency graph updates with a function $B: S^{i+1} \times \mathbb{N} \to \mathcal{P}(\mathbb{N})$:
\begin{align}
    G_{i + 1} = G(q_1, \ldots, q_i \mid s_1, \ldots, s_{i+1})
    \notag\\= \{q_k \mid k \in B(s_1,\ldots,s_{i+1}, i+1)\}
\end{align}
\vspace{-0.1em}
where $B(s_1,\ldots,s_i, j)$ returns the set of indices that determine the dynamic dependencies based on the current state sequence and global semantic constraints.
\vspace{-0.4em}
\item \textbf{State Transition Function} $F: \mathcal{P}(Q^i) \times Q^i \to q$ computes the next state using selected predecessors from the prefix subsequence $Q^i$:
\begin{equation*}
    \begin{split}
        q_{i+1} &= F(G(q_1, \ldots, q_i \mid s_1, \ldots, s_{i+1}), s_{i+1}) \\
        &\text{where} \quad F(\emptyset) = \text{Constant}
    \end{split}
\end{equation*}
\vspace{-1em}
\item \textbf{Final Result Operation:} The system's final result is computed through a sequence of operations $L_i$ that aggregate states progressively: $L: q \times q \rightarrow R$
\begin{equation*}
\begin{split}
    L_1 &= H(\emptyset, q_1) = q_1 \\
    L_i &= H(L_{i-1}, q_i) \quad \text{for } i = 2,\ldots,N
\end{split}
\end{equation*}
\vspace{-0.2em}
where $H$ is an aggregation function like maximum, minimum, or summation function and $H(\emptyset, q) = q$ that can be specialized as needed.

\end{definition}
\vspace{-\baselineskip}

\subsection{Recap Conditions for Effective Chain of Thought} 
\vspace{-0.7em}
\label{sec:recap-theory}

Before introducing the Chain of Thought \cite{wei2022chain} approach, we establish two fundamental conditions shown in Fig~\ref{fig:recap_main} that enable more effective reasoning chains. We emphasize that our analysis in this section rests on \emph{finite floating-point precision constraints} of practical implementations rather than on idealized theoretical attention limits.\newline
1. \textbf{Outside Window Recap Condition:} Due to finite precision constraints, only tokens with the top-k attention scores can be effectively recalled. When the distance between the target token position and current position exceeds a threshold, the attention scores become negligible within the given precision, preventing recall of tokens outside this window. This necessitates strategic recapping of tokens within the window to maintain information flow. See Appendix \ref{thm: Recap 1} for formal analysis.
\newline
2. \textbf{Inside Window Recap Condition:} For optimal model performance, tokens within the attention window should be organized following the ground truth causal order, avoiding irrelevant intermediary tokens that could impede convergence. As proven in Appendix \ref{thm: Recap 2}, including irrelevant tokens in the sequence slows down model convergence. Therefore, we must carefully structure the token sequence within the window (e.g., through techniques like reverse ordering or selective token repetition) to align with the underlying causal structure.
\tikzset{
    basic/.style={font=\sffamily, align=center},
    token/.style={
        basic,
        draw=gray!40,
        fill=gray!15,
        rounded corners=3pt,
        minimum width=0.8cm,
        minimum height=0.8cm,
        inner sep=1pt,
        drop shadow={opacity=0.08, shadow xshift=1pt, shadow yshift=-1pt},
        font=\small\sffamily
    },
    highlight/.style={token, draw=red!80, fill=red!5, line width=1pt},
    copied/.style={token, draw=teal!80, fill=teal!10, line width=1pt},
    window_box/.style={draw=blue!60, line width=1.5pt, fill=blue!5, rounded corners=5pt},
    dep_arrow/.style={->, >={Stealth[length=2mm]}, draw=red!70, line width=0.8pt, dashed, shorten >=2pt, shorten <=2pt},
    recap_arrow/.style={->, >={Latex[length=3mm]}, draw=teal!80!black, line width=1.2pt, opacity=0.6, shorten >=3pt, shorten <=3pt},
    llm_box/.style={draw=blue!40!black, top color=blue!5, bottom color=white, rounded corners, minimum width=6cm, minimum height=1cm, font=\bfseries\sffamily, drop shadow},
    brace_style/.style={decorate, decoration={calligraphic brace, amplitude=8pt, raise=5pt}, very thick, pen colour={gray!80}}
}
\begin{figure}[h]
    \centering
    
    \subfigure[\textbf{Outside Window Recap:} $s_1$ is copied into the window. \label{fig:outside_recap}]{%
        \resizebox{0.48\textwidth}{!}{%
        \begin{tikzpicture}[node distance=1.5cm]
            \node[basic, anchor=east, font=\bfseries] at (0, 2.8) {Before Recap};
            \node[basic, anchor=west, font=\footnotesize, text=blue!80] at (9, 2.8) {Window (Size 5)};

            \foreach \i/\t in {1/s1, 2/s2, 3/s3, 4/s4, 5/s5, 6/s6, 7/s7, 8/s8} {
                \node[token] (b\i) at (\i*1.1, 2.3) {\t};
            }
            \node[highlight] at (b1) {s1}; 
            \node[highlight] at (b5) {s5};

            \begin{scope}[on background layer]
                \node[window_box, fit=(b3)(b7)] (win_before) {};
            \end{scope}

            \draw[dep_arrow, bend left=50] (b1.north) to node[midway, above=1pt, font=\tiny, text=red] {dependency} (b5.north);

            \node[basic, anchor=east, font=\bfseries] at (0, 0) {After Recap};
            
            \node[token] (a1) at (1.1, -0.5) {s1};
            \node[token] (a2) at (2.2, -0.5) {s2};
            \node[token] (a3) at (3.3, -0.5) {s3};
            \node[token] (a4) at (4.4, -0.5) {s4};
            \node[copied] (a_copy) at (5.5, -0.5) {s1};
            \node[highlight] (a5) at (6.6, -0.5) {s5};
            \node[token] (a6) at (7.7, -0.5) {s6};
            \node[token] (a7) at (8.8, -0.5) {s7};
            \node[token] (a8) at (9.9, -0.5) {s8}; 

            \begin{scope}[on background layer]
                \node[window_box, fit=(a4)(a7)] (win_after) {};
            \end{scope}

            \draw[recap_arrow, out=-90, in=110] (b1.south) to node[midway, fill=white, inner sep=0pt, font=\tiny, text=teal] {copy} (a_copy.north);
            \draw[dep_arrow, bend left=50] (a_copy.north) to (a5.north);

            \draw[brace_style] (a8.south east) -- (a1.south west) node[midway, below=20pt] (brace_center) {};
            \node[llm_box, below=0.3cm of brace_center] (llm) {Large Language Model};
            \draw[
              ->,
              line width=1.6pt,
              black!65,
              >={Stealth[length=4mm]},
              line cap=round,
              shorten >=2pt,
              shorten <=2pt
            ]
            ([yshift=3mm]brace_center.center) -- (llm.north);
        \end{tikzpicture}%
        }%
    }
    \hfill
    \subfigure[\textbf{Inside Window Recap:} $s_1$ and $s_3$ are copied closer to $s_6$. \label{fig:inside_recap}]{%
        \resizebox{0.48\textwidth}{!}{%
        \begin{tikzpicture}[node distance=1.5cm]
            \node[basic, anchor=east, font=\bfseries] at (0, 2.8) {Before Recap};
            \node[basic, anchor=west, font=\footnotesize, text=blue!80] at (9, 2.8) {Window (Size 11)};

            \foreach \i/\t in {1/s1, 2/s2, 3/s3, 4/s4, 5/s5, 6/s6, 7/s7, 8/s8} {
                \node[token] (bb\i) at (\i*1.1, 2.3) {\t};
            }
            \node[highlight] at (bb1) {s1};
            \node[highlight] at (bb3) {s3};
            \node[highlight] at (bb6) {s6};

            \begin{scope}[on background layer]
                \node[window_box, fit=(bb1)(bb8)] (win_b_right) {};
            \end{scope}

            \draw[dep_arrow, bend left=60] (bb1.north) to (bb6.north);
            \draw[dep_arrow, bend left=45] (bb3.north) to (bb6.north);

            \node[basic, anchor=east, font=\bfseries] at (0, 0) {After Recap};

            \node[token] (aa1) at (1.1, -0.5) {s1};
            \node[token] (aa2) at (2.2, -0.5) {s2};
            \node[token] (aa3) at (3.3, -0.5) {s3};
            \node[token] (aa4) at (4.4, -0.5) {s4};
            \node[token] (aa5) at (5.5, -0.5) {s5};
            \node[copied] (aa_c1) at (6.6, -0.5) {s1};
            \node[copied] (aa_c3) at (7.7, -0.5) {s3};
            \node[highlight] (aa6) at (8.8, -0.5) {s6};
            \node[token] (aa7) at (9.9, -0.5) {s7};
            \node[token] (aa8) at (11.0, -0.5) {s8}; 

            \begin{scope}[on background layer]
                \node[window_box, fit=(aa1)(aa6)] (win_a_right) {};
            \end{scope}

            \draw[recap_arrow, out=-100, in=130] (bb1.south) to (aa_c1.north);
            \draw[recap_arrow, out=-50, in=110] (bb3.south) to (aa_c3.north);

            \draw[dep_arrow, bend left=60] (aa_c1.north) to (aa6.north);
            \draw[dep_arrow, bend left=45] (aa_c3.north) to (aa6.north);

            \draw[brace_style] (aa8.south east) -- (aa1.south west) node[midway, below=20pt] (brace_center2) {};
            \node[llm_box, below=0.3cm of brace_center2] (llm2) {Large Language Model};
            \draw[
              ->,
              line width=1.6pt,
              black!65,
              >={Stealth[length=4mm]},
              line cap=round,
              shorten >=2pt,
              shorten <=2pt
            ]
            ([yshift=3mm]brace_center2.center) -- (llm2.north);
        \end{tikzpicture}%
        }%
    }
    
\caption{\textbf{Recap mechanism visualization.}
Selected tokens are copied to restructure the input sequence so that relevant dependencies fall within the effective attention window before being processed by the LLM.}
    \label{fig:recap_main}
\end{figure}
Combined with these two recap conditions, for each inference step in a compound task, we must ensure all dependent tokens are both compactly positioned and arranged in their causal order. This strategy enables effective information flow while maintaining computational efficiency. In implementation, since tokens may have multiple dependencies, it is preferable to copy tokens rather than move them.

\subsection{CoT Captures the Dynamic State Transitions}
\vspace{-0.7em}
Inspired by recent work \cite{li2024chainthoughtempowerstransformers, feng2023revealingmysterychainthought} and building upon our recap conditions, we now introduce a CoT formalization specifically designed to capture the dynamic state transitions in compound tasks. Figure \ref{fig:disbution_shift} conceptually illustrates how CoT helps to generalize. Rather than allowing the model to learn superficial shortcuts, our CoT format explicitly tracks the evolving dependencies and state transitions that characterize the ground truth recurrent solution.

\begin{definition}[Chain of Thought]
  \vspace{0em}
    \small\begin{align*}
    \vspace{-0.5em}
        \text{Input:} & \; s_1 \mid \cdots \mid s_N \\[1ex]
        \text{CoT Steps:} & \; \langle\text{sep}\rangle \; s_2 \mid G_2 \mid q_2 \mid L_1 \mid L_2 \\
        & \; \langle\text{sep}\rangle \; \cdots \\
        & \; \langle\text{sep}\rangle \; s_{N} \mid G_N \mid q_N \mid L_{N-1} \mid L_N
    \end{align*}
    \label{def:CoT}
\end{definition}
\vspace{-\baselineskip}
This format structures each step to track dependencies ($G_i$), states ($q_i$), and intermediate results ($L_i$), guiding the model toward learning true solution dynamics. Remarkably, this principled approach can be implemented with minimal architectural requirements:

\begin{proposition}
\label{thm:CoT}
For any compound problem satisfying Definition \ref{def:CP}, and for any input length bound $n \in \mathbb{N}$, there exists an autoregressive Transformer with: 1.Constant depth $L$, 2.Constant hidden dimension $d$, 3.Constant number of attention heads $H$
where $L$, $d$, and $H$ are independent of $n$, such that the Transformer correctly generates the Chain-of-Thought solution defined in Definition \ref{def:CoT} for all input sequences of length at most $n$. Furthermore, all parameter values in the Transformer are bounded by $O(\text{poly}(n))$.
\end{proposition}


\section{Experiment}
\vspace{-0.7em}
In order to investigate how data collection methods impact LMs' generalization ability, we conducted controlled studies on three compound reasoning tasks with systematic distribution shifts. Our experiments address three key questions: (1) How does direct QA training compare to CoT in OOD generalization? (2) Can increasing data help direct QA training to generalize to OOD data? (3) Can repeat the conditions for each subtask of one compound task help generalization?
\subsection{Setting}
\vspace{-0.7em}
\subsubsection{Task and Experiment setting}
\vspace{-0.7\baselineskip}
We evaluate the generalization ability of LMs through three compound reasoning tasks:
(1) Longest Increasing Subsequence (LIS),
(2) Multi-Step Path Counting (MPC), and
(3) Equation Restoration with Variable Computation (ERVC).
Representative examples for all tasks are provided in Appendix
\ref{app:LIS}, \ref{app:MPC}, and \ref{app:ERVC}.
\vspace{-0.7\baselineskip}

\paragraph{Longest Increasing Subsequence (LIS):}
Given an input sequence $X^n$ of integers, compute the length of the longest increasing subsequence.
The model is trained on sequences of complexity levels $n_1=4$ and $n_2=16$, and evaluated on
$n_3=10$, testing generalization across sequence lengths and positional dependencies.
\vspace{-0.7\baselineskip}

\paragraph{Multi-Step Path Counting (MPC):}
Given an input sequence $X^n$ specifying available steps (0 for forbidden, 1 for available),
compute the number of unique paths to reach position $n$ using steps of size $\{1,2,3\}$.
Training uses $n_1=20$ and $n_2=40$, and evaluation uses $n_3=30$, testing generalization
under varying combinatorial constraints.
\vspace{-0.7\baselineskip}

\paragraph{Equation Restoration and Variable Computation (ERVC):}
Given observations of $n$ variables and their values, recover $m$ underlying linear equations,
then compute target variables under new assignments.
Training and testing involve distribution shifts in both $n$ and $m$.

\subsubsection{Training Dataset Preparation}
\vspace{-0.7\baselineskip}
For each task, we created datasets in two formats:
\textbf{1. Question-Answer (Q-A):} Questions paired directly with final answers, representing supervision without reasoning steps (CoT-0\%).
\textbf{2. Question-Chain-of-Thought (Q-CoT):} Questions paired with step-by-step CoT explanations leading to answers, providing process supervision (CoT-100\%).
To investigate the impact of partial CoT supervision, we also conducted ablation studies using probabilistic CoT dropout. In these settings, each step within a complete CoT demonstration has a probability of being retained in the training data. We explored dropout rates of 30\%, 50\%, and 70\% (CoT-30\%, CoT-50\%, CoT-70\%).  We focus on probabilistic dropout as preliminary experiments with fixed-portion CoT dropout showed significant OOD performance degradation.

For scaling experiments, we varied the dataset size from 1k to 30k samples for both Q-A and Q-CoT regimes. Models were trained on the ID data splits and evaluated on both ID and OOD splits to measure generalization gaps.

\begin{figure}
    \centering
    \subfigure[Multi-Step Path Count (MPC)]{
        \includegraphics[width=0.9\linewidth]{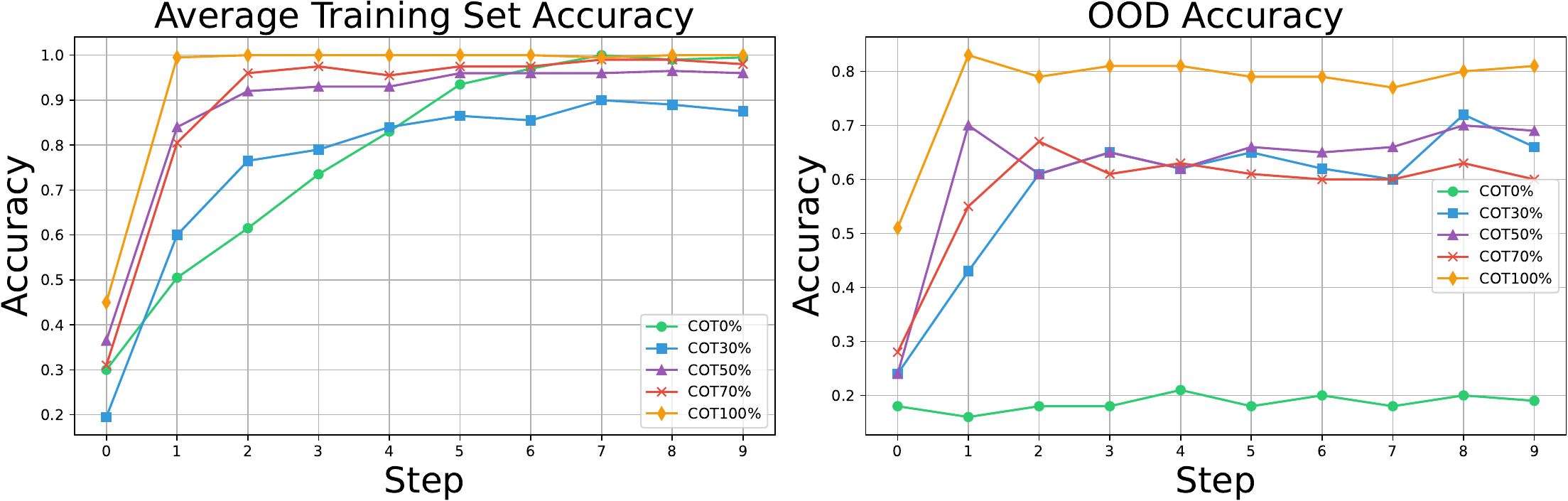}
        \label{fig:ratio_3step}
    }\vfill
    \vspace{-0.5em}
    \subfigure[Longest Increasing Sequence (LIS)]{
        \includegraphics[width=0.9\linewidth]{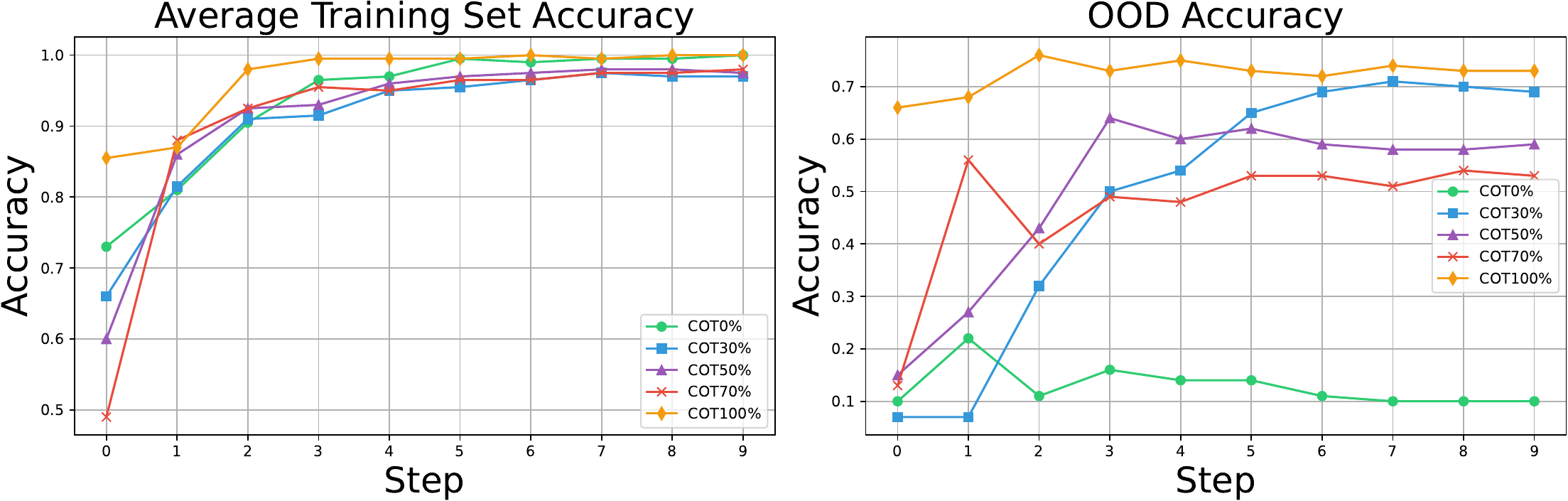}
        \label{fig:lis_analysis}
    }\vfill
    \vspace{-0.5em}
    \subfigure[Equation Restoration and Variable Computation (ERVC21-43)]{
        \includegraphics[width=0.9\linewidth]{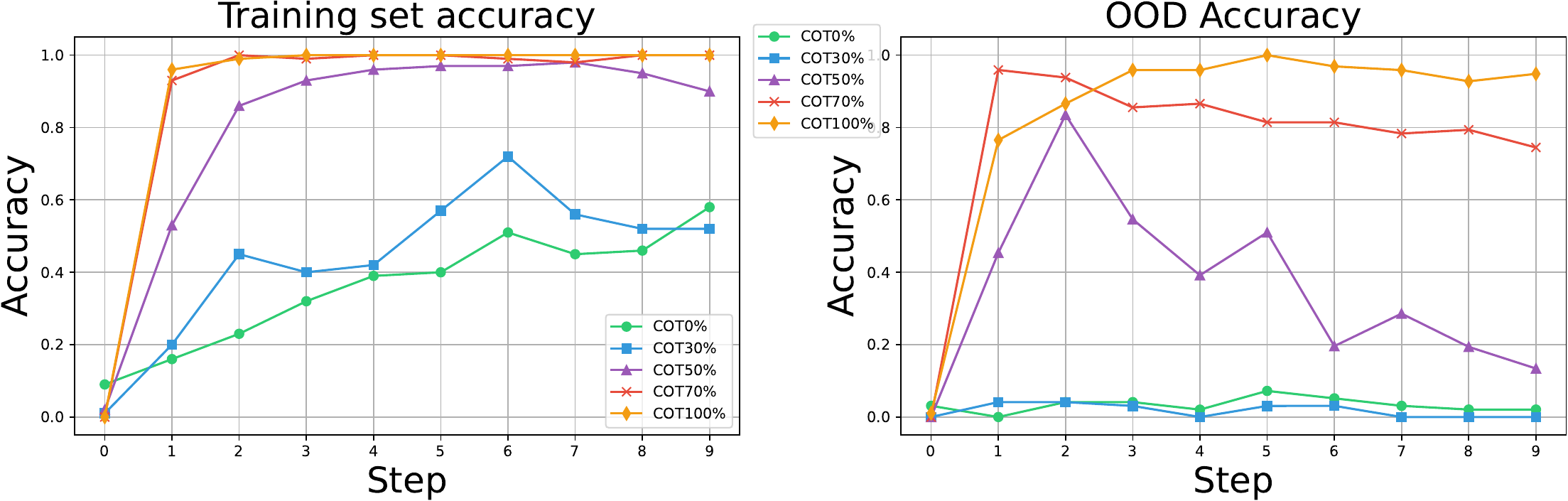}
        \label{fig:dec24}
    }\vfill
    \vspace{-0.5em}
    \subfigure[Equation Restoration and Variable Computation (ERVC21-41)]{
        \includegraphics[width=0.9\linewidth]{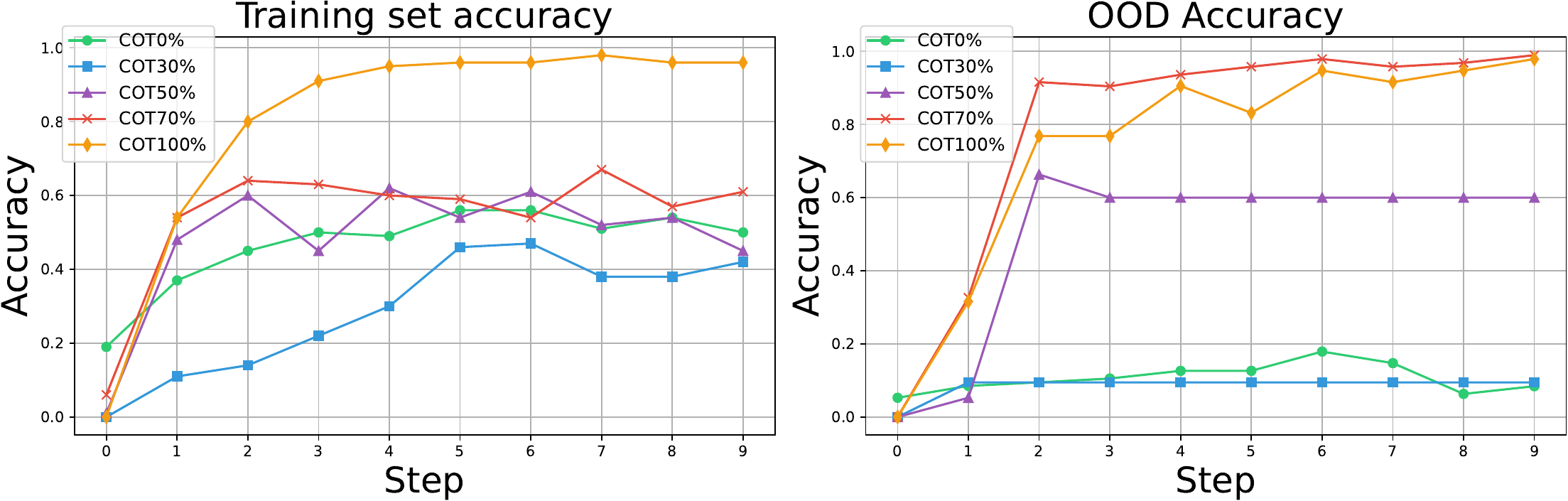}
        \label{fig:dec14}
    }
    \vspace{-0.7\baselineskip}
    \caption{Impact of CoT granularity on In-Distribution and Out-of-Distribution performance across four tasks. Left panels show ID accuracy, right panels show OOD accuracy. While non-CoT models achieve high ID accuracy, they demonstrate significantly lower OOD generalization performance.}
    \vspace{-2em}
    \label{fig:combined}
\end{figure}

\subsubsection{Model Training and Evaluation}
\vspace{-0.7\baselineskip}
For the Longest Increasing Subsequence and Multi-Step Path Counting tasks, we implemented a 6-layer transformer architecture trained from scratch, featuring an embedding size of 256/512 and 16 attention heads. The Equation Restoration tasks were approached differently, utilizing the Phi-3.5-mini-instruct model as the foundation for task-specific fine-tuning. All model variants incorporate Rotary Position Embedding.
In evaluating model performance, we employed different strategies for Q-A and Q-CoT data. For Q-A data in both the Longest Increasing Subsequence and Multi-Step Path Counting tasks, correctness is determined by comparing the token immediately preceding the \texttt{<eos>} token. For Q-CoT data, our evaluation extends beyond the final answer, incorporating intermediate computational tokens (such as $Q$ and $L$) that contribute to the solution. This comprehensive evaluation approach helps mitigate cases where incorrect reasoning processes might accidentally yield correct final answers through guessing.

\subsection{Results Analysis}

\begin{table}[htbp]
\vspace{-1em}
\caption{Performance comparison across MPC, LIS, ERVC tasks, where ID denotes the in-distribution task, OOD denotes out-of-distribution task, ID-OOD denotes the different between ID and OOD performance. }
    \centering
    \resizebox{0.9\columnwidth}{!}{
        \begin{tabular}{ccccccccc}
            \hline
            \multicolumn{1}{c}{\multirow{2}{*}{Task}} & \multicolumn{2}{c}{MPC} & \multicolumn{2}{c}{LIS} & \multicolumn{2}{c}{ERVC21-43} & \multicolumn{2}{c}{ERVC21-41} \\ \cline{2-9} 
            \multicolumn{1}{c}{} & Q-A & Q-CoT & Q-A & Q-CoT & Q-A & Q-CoT & Q-A & Q-CoT \\ \hline
            ID $\uparrow$ & 0.83 & 1.0 & 0.97 & 0.995 & 0.39 & 1.0 & 0.49 & 0.95 \\
            OOD $\uparrow$ & 0.21 & 0.81 & 0.14 & 0.75 & 0.02 & 0.959 & 0.126 & 0.905 \\
            ID-OOD $\downarrow$ & 0.62 & 0.19 & 0.83 & 0.245 & 0.37 & 0.041 & 0.364 & 0.045 \\ \hline
        \end{tabular}
    }
    \label{tab:performance gap}
     \vspace{-1em}
\end{table}

\subsubsection{Generalization Gap and Scaling Curves}
\begin{figure}[]
\centering
\subfigure[Data Scaling Curves in MPC task]{
    \includegraphics[width=0.8\linewidth]{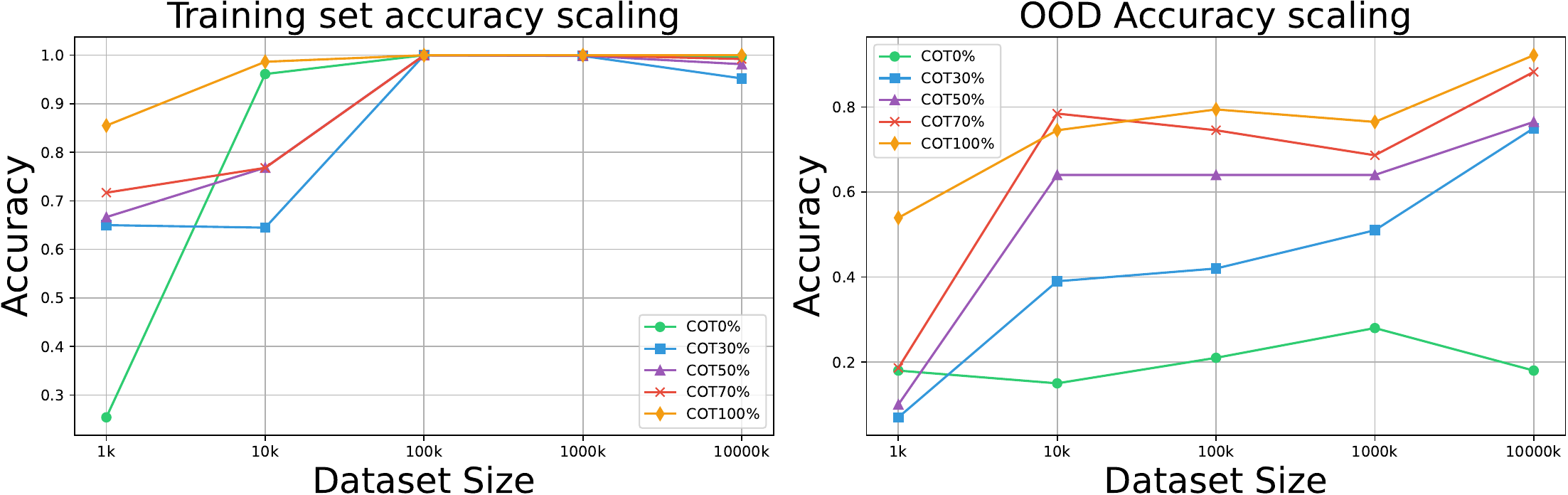}
}\hfill
\subfigure[Data Scaling Curves in LIS task]{
    \includegraphics[width=0.8\linewidth]{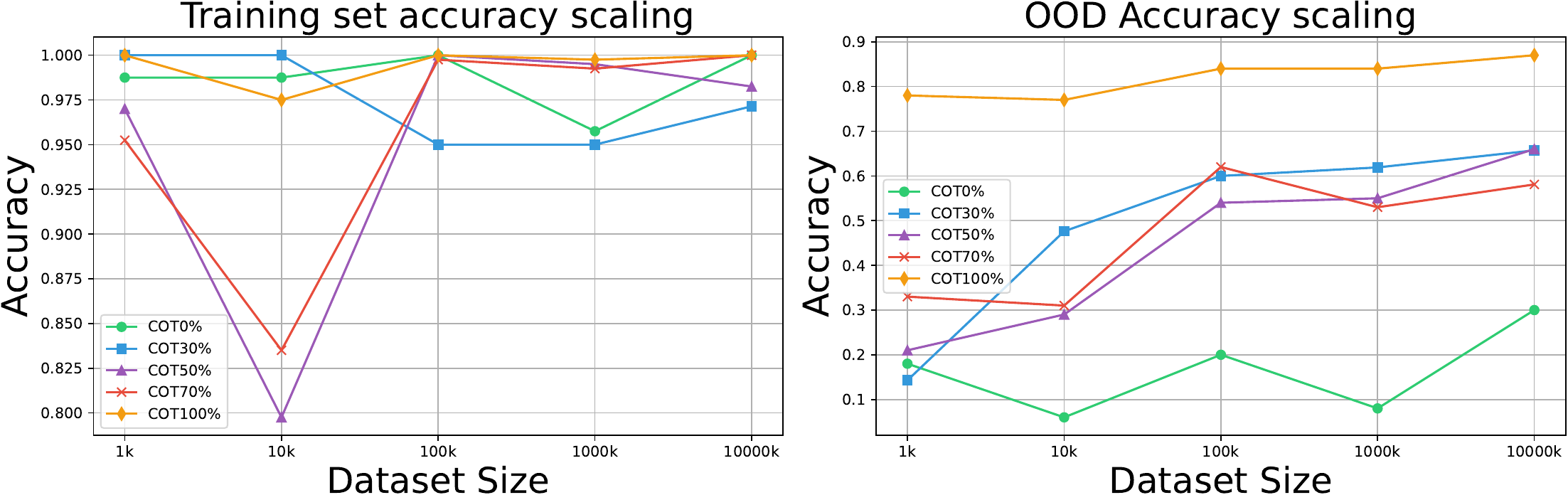}
    \label{fig:lis}
}   
\vspace{-0.7\baselineskip}
\caption{The illustration of the impact of data scale on ID and OOD performance in the MPC and LIS tasks. Left: ID performance. Right: OOD performance. We can clearly see that even when the training data increases to 10M, LMs trained without CoT can exhibit near-perfect performance on training tasks, yet they still achieve only about ~30\% accuracy on OOD data.}
\label{fig:scaling_results}
\end{figure}
\vspace{-0.7em}
Figure \ref{fig:scaling_results} and Table \ref{tab:performance gap} show accuracy scaling curves for MPC and LIS as dataset size increases, comparing CoT ratios (0\%, 30\%, 50\%, 70\%, 100\%). Consistent with Theorem \ref{thm:kl-reduction} in Appendix~\ref{app:shift_ana}, models trained with higher CoT (70–100\%) achieve better OOD generalization than direct QA (0\%), with faster gains and higher sustained accuracy.

For Q-A (0\%), OOD performance plateaus quickly, revealing poor generalization. In contrast, Q-CoT models continue improving with scale, reducing the gap.
Partial CoT (30–50\%) degrades OOD performance, supporting the claim that incomplete reasoning chains weaken prefix coverage, increase KL divergence, and harm generalization.

Overall, the results validate the theory: CoT supervision improves OOD generalization by encoding reasoning steps, effectively transforming OOD problems into in-distribution-like ones.

\subsubsection{Granularity-Generalization Tradeoff}
\vspace{-0.7em}
Figure \ref{fig:combined} demonstrates a clear relationship between CoT step coverage and OOD generalization performance across tasks. Specifically, models trained with more complete CoT steps exhibit superior generalization capabilities on OOD samples. To formally characterize this empirical observation, we develop a theoretical framework in with a simple linear model that quantitatively explains how the omission of CoT steps during training systematically degrades model performance. Our analysis reveals that the degradation follows an exponential decay pattern with respect to the number of missing steps, highlighting the critical importance of maintaining granular reasoning steps for robust generalization. 

\subsubsection{Data Scaling Experiments}
\vspace{-0.7em}
To investigate the effect of training data scaling on generalization, we examined the performance trends as dataset size increased. As depicted in Figure~\ref{fig:scaling_results}, a key observation emerges: models trained without CoT supervision exhibit limited improvement in OOD generalization despite substantial increases in training data size (up to 10,000k samples). While ID accuracy improves with larger datasets for these models, their OOD performance plateaus, indicating that even the increased data cannot help generalization to novel compound tasks with direct Q-A training. In stark contrast, models trained with CoT supervision demonstrate a distinct capacity to translate larger datasets into improved OOD generalization. 

\subsubsection{Sample Efficiency}
\vspace{-0.7em}
As shown in Figure \ref{fig:scaling_results}, models trained with a higher percentage of CoT examples demonstrate superior sample efficiency in both MPC and LIS scaling experiments. Specifically, models with $70\%$ and $100\%$ CoT (CoT-$70\%$ and CoT-$100\%$) achieve a higher OOD accuracy with significantly smaller dataset sizes compared to models trained with lower CoT percentages. This effect is particularly pronounced in the small data regime (1k-10k samples), where CoT-$100\%$ maintains a relatively robust OOD performance (0.55-0.75 accuracy) while models with less CoT supervision struggle (below 0.4 accuracy). This suggests that incorporating more reasoning steps through CoT not only improves overall performance but also enables more efficient learning from limited training data.
\subsubsection{Recap Condition Ablation}
\begin{table}[h]
\vspace{-1em}
\caption{Model performance comparisons on the with and without recap condition on the Equation Restoration Task}
 \vspace{-0.5em}
    \centering
    \resizebox{0.7\columnwidth}{!}{ %
\begin{tabular}{ccccc}
\hline
Task                & \multicolumn{2}{c}{ERVC21-43} & \multicolumn{2}{c}{ERVC21-41} \\ \hline
Dataset               & Training     & OOD       & Training     & OOD       \\ \hline
w/ recap condition  & 0.974        & 0.947     & 0.997        & 0.952     \\
w/o recap condition & 0.436        & 0.126     & 0.638        & 0.558     \\ \hline
\end{tabular}
 }
\label{tab:model-comparison}
\vspace{-1.em}
\end{table}
We investigate the impact of the recap condition on model performance in Equation Restoration tasks in Table \ref{tab:model-comparison}. For ERVC21-43, models trained with the recap condition achieve a high training accuracy of 0.974 and maintain this level of performance on the OOD set (0.974).  In contrast, models trained without the recap condition exhibit significantly lower accuracy on both the training set (0.436) and the OOD set (0.126). A similar trend is observed for ERVC21-41. The model with the recap condition achieves near-perfect training accuracy (0.997) and strong OOD performance (0.952).  Conversely, the model without the recap condition shows substantially reduced accuracy in both training (0.638) and OOD settings (0.558). These results strongly suggest that the recap condition of each subtask plays a critical role in the generalization of LMs and and validate our theoretical propositions outlined in Section \ref{sec:recap-theory}

\section{Limitation}
\vspace{-0.7em}
Limitations We acknowledge three limitations in our study. First, to rigorously isolate the impact of reasoning density, we rely on synthetic tasks; quantifying "granularity" in unstructured natural language remains an open challenge. Second, there is a distinct trade-off in efficiency: while finer-grained CoT achieves sample efficiency (training parsimony), it inevitably increases inference costs (compute parsimony) due to longer sequence generation. Third, we focus on fundamental Transformer dynamics; while our theoretical analysis suggests our findings are structurally invariant, we have not empirically verified these specific scaling laws on massive 100B+ parameter foundation models.

\section{Conclusion}
\vspace{-0.7em}
In this paper, through controllable systematic experimentation and theoretical analysis in compound tasks, we demonstrate that (1) Direct QA-trained models exhibit dramatic performance degradation under shifts despite high in-distribution accuracy with even 100k data in the single task, (2) CoT granularity directly governs generalization strength, with fine-grained reasoning steps enabling higher OOD accuracy, and (3) CoT training achieves comparable performance to QA paradigms with much less data, underscoring its sample efficiency. Theoretically, we prove that CoT enforces structured reasoning paths, a process amplified by transformer positional embeddings through subtask condition recurrence, and these trends are visually corroborated in Appendix Figure~\ref{fig:ood_detail}, where models trained with 100\% CoT maintain high accuracy across varying sequence lengths while direct QA models collapse under OOD conditions despite strong in-domain performance. These findings provide a reliable guide for data collection practices when leveraging LMs for novel tasks. We will explore automated CoT generation and dynamic granularity adaptation to further reduce annotation costs in the future work.

\newpage
\bibliography{reference}

\newpage
\appendix
\onecolumn
\section{Explain Compound task in formal definition}
As shown in the figure, original tree like data structure can be converted to an array which is easily feeded into language model.
\begin{figure}
    \centering
    \includegraphics[width=1\linewidth]{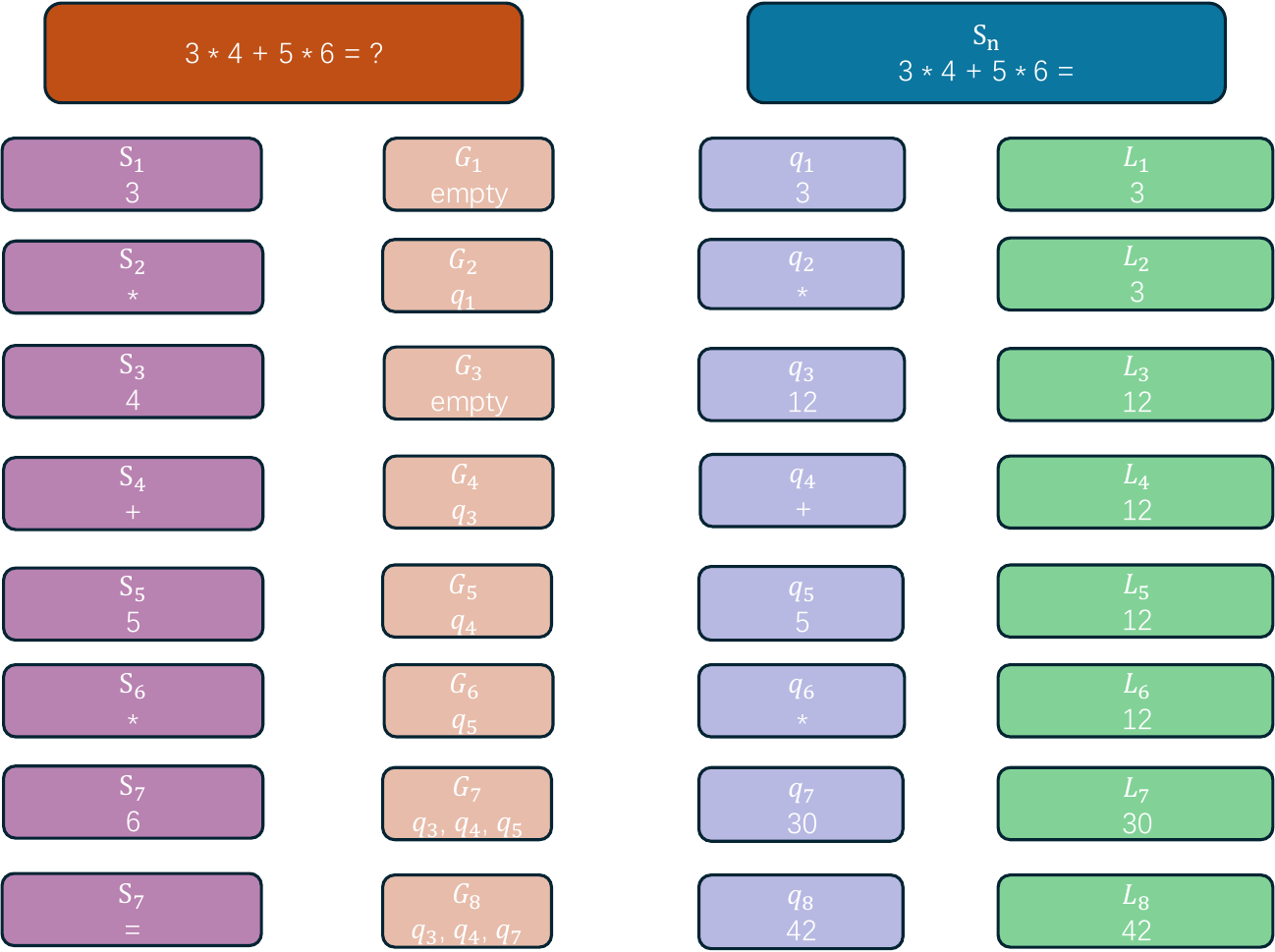}
    \caption{Step by Step explanation on definition \ref{def:CP}}
    \label{fig:cp2}
\end{figure}
\section{Code and Data}
We provide an anonymous page, you can following the instruction to generate data, training model and evaluation.
\url{https://github.com/physicsru/Scaling-Curves-of-CoT-Granularity-for-Language-Model-Generalization}

\section{Recap Condition Analysis}

\begin{theorem}[Outside Token Recap Condition under RoPE]
\label{thm: Recap 1}
Consider a transformer with Rotary Positional Embedding (RoPE) using angles $\theta_j = 10000^{-2j/d_{\text{model}}}$. Given finite computational precision $s$ and minimum resolvable attention score $\epsilon$, there exists a threshold distance $\tau > 0$ such that for all positional distances $d > \tau$:
\[
|A(d)| < \epsilon,
\]
where $A(d)$ is the attention score between tokens at distance $d$. Consequently, tokens beyond $[i_{\text{current}} - \tau, i_{\text{current}} + \tau]$ cannot be recalled.

\begin{proof}
\noindent Step 1: Attention Score Formulation
The RoPE attention score between positions $m$ and $n$ (distance $d = |m-n|$) is:
\[
A(d) = \text{Re}\left[\sum_{j=0}^{d_{\text{model}}/2-1} h_j e^{\mathbf{i}d\theta_j}\right], \quad h_j := q_{[2j:2j+1]}\mathbf{k}^*_{[2j:2j+1]}
\]
where $h_j$ encodes query-key interactions for the $j$-th dimension pair.

\noindent Step 2: Abel Transformation\cite{men2024baseropeboundscontext}
Let $S_{j+1} = \sum_{k=0}^j e^{\mathbf{i}d\theta_k}$ with $S_0 = 0$. Using summation by parts:
\[
\sum_{j=0}^{d_{\text{model}}/2-1} h_j e^{\mathbf{i}d\theta_j} = \sum_{j=0}^{d_{\text{model}}/2-1} (h_j - h_{j+1}) S_{j+1}
\]
Taking absolute values:
\[
|A(d)| \leq \sum_{j=0}^{d_{\text{model}}/2-1} |h_{j+1} - h_j| \cdot |S_{j+1}|
\]

\bigskip
\noindent Step 3: Bounding Query-Key Differences.
Assume that the query and key representations are uniformly bounded so that $\|q\|,\,\|k\| \le C$. In particular, since each $h_j$ results from a dot product between sub-vectors from $q$ and $k$, we have $|h_j| \le C^2$. Moreover, if we assume that the embeddings vary smoothly with the index $j$ (as expected from the continuity of underlying network nonlinearities and weight matrices), then the mean value theorem implies that the difference
\[
|h_{j+1} - h_j|
\]
is bounded by a Lipschitz constant. That is, there exists a constant $M = \mathcal{O}(C^2)$ such that for every $j$
\[
|h_{j+1} - h_j| \le M.
\]
A more refined analysis might track this difference in terms of the network’s smoothness, but the key point is that each difference is uniformly bounded by a constant depending on $C$.

\bigskip
\noindent Step 4: Analyzing Oscillatory Sums.
We now study the partial sum
\[
S_{j+1} = \sum_{k=0}^j e^{\mathbf{i}d\theta_k}.
\]
Two regimes are considered:

\emph{(i) Low-frequency regime ($k \le j_0$):}  
For sufficiently small indices $k$, we have $\theta_k$ being relatively large so that
\[
d\theta_k \gg 1.
\]
In this regime the phases $e^{\mathbf{i}d\theta_k}$ change rapidly with $k$, leading to cancellations among the terms. A standard bound for such oscillatory sums yields
\[
\Big|\sum_{k=0}^{j_0} e^{\mathbf{i}d\theta_k}\Big| \le \frac{2}{\left|1 - e^{\mathbf{i}d\theta_0}\right|} = \mathcal{O}(1),
\]
since the denominator remains bounded away from zero when $d\theta_0$ is large.

\emph{(ii) High-frequency regime ($k > j_0$):}  
For larger $k$, the angle $\theta_k$ becomes very small (because $\theta_k$ decays exponentially with $k$), so that $d\theta_k \ll 1$. In this case we can use the first-order Taylor expansion:
\[
e^{\mathbf{i}d\theta_k} \approx 1 + \mathbf{i}d\theta_k.
\]
Thus, for indices $j>j_0$, the sum becomes approximately
\[
\sum_{k=j_0+1}^j e^{\mathbf{i}d\theta_k} \approx (j-j_0) + \mathbf{i}d\sum_{k=j_0+1}^j \theta_k.
\]
While the real part grows (almost) linearly in the number of terms, the alternating phases and the small magnitude of $d\theta_k$ imply that additional cancellation occurs when the two regimes are combined. In the worst case one may bound
\[
|S_{j+1}| \le \mathcal{O}(\log d)
\]
by appealing to harmonic series type estimates; this bound is loose but captures the fact that the cancellation improves with larger $d$.

\bigskip
\noindent Step 5: Combining the Results. 
Returning to the bound from Step~2, we have
\[
|A(d)| \le \sum_{j=0}^{d_{\text{model}}/2-1} |h_{j+1} - h_j|\, |S_{j+1}|
\]
and using the bounds from Steps~3 and 4,
\[
|A(d)| \le M \sum_{j=0}^{d_{\text{model}}/2-1} \mathcal{O}(\log d) = \mathcal{O}(M\, d_{\text{model}} \log d).
\]
In practice, the alternating signs in the summands produce even stronger decay in $d$, and empirical observations suggest a polynomial decay of the form
\[
|A(d)| \le \mathcal{O}\Big(\frac{M}{\sqrt{d}}\Big).
\]

\bigskip
\noindent Step 6: Precision Threshold.
For a given minimum resolvable attention score $\epsilon$, we require
\[
\frac{M}{\sqrt{d}} < \epsilon \quad \Longrightarrow \quad d > \left(\frac{M}{\epsilon}\right)^2.
\]
Defining
\[
\tau = \left(\frac{M}{\epsilon}\right)^2,
\]
we conclude that for any $d > \tau$, it holds that $|A(d)| < \epsilon$. That is, tokens beyond the window indexed by $[i_{\text{current}}-\tau,\, i_{\text{current}}+\tau]$ effectively contribute an attention score below the resolvable threshold.

\end{proof}

\textbf{Interpretation:}
\begin{itemize}
\item The $\theta_j$ schedule creates frequency-dependent decay: high frequencies (small $j$) attenuate rapidly
\item Window size $\tau \propto (M/\epsilon)^2$ explains memory limitations in long contexts
\item Practical implementations must balance $d_{\text{model}}$ and precision $s$ for optimal $\tau$
\end{itemize}
\end{theorem}

\begin{theorem}[Inside window recap condition under Rope]
\label{thm: Recap 2}
Assume: The true function is $s_2 = {w^{\circ}}^{\top} e(s_1) + \epsilon$, with $s_2 \perp s_x$ in expectation.
Consider two linear models:
\begin{itemize}
\item $M_s$: $y_s = {w_s}^{\top} e(s_1)$
\item $M_l$: $y_l = {w_1}^{\top} e(s_1) + {w_2}^{\top} e(s_x)$
\end{itemize}
Under mean-squared-error training:
\begin{itemize}
\item $g_s = \frac{\partial L_s}{\partial w_s}$
\item $(g_1, g_2) = \left(\frac{\partial L_l}{\partial w_1}, \frac{\partial L_l}{\partial w_2}\right)$
\end{itemize}
Then in finite-data regimes or with random noise, the gradient component of $g_1$ along $(w^{\circ} - w_1)$ is typically smaller than the corresponding component of $g_s$ along $(w^{\circ} - w_s)$
Formally,
\[
\mathbb{E}\left[ \left\langle \mathbf{g}_1, \mathbf{w}^\circ - \mathbf{w}_1 \right\rangle \right] < \mathbb{E}\left[ \left\langle \mathbf{g}_s, \mathbf{w}^\circ - \mathbf{w}_s \right\rangle \right],
\]
leading to slower convergence for \( \mathcal{M}_{l} \) when \( s_{x} \) is irrelevant.
\end{theorem}

\begin{proof}
We analyze the gradients of both models and demonstrate how irrelevant features in \( \mathcal{M}_l \) reduce the effective gradient signal.

\noindent Step 1: Express Gradients for Both Models

For model \( \mathcal{M}_s \), the loss is:
\[
L_s = \mathbb{E}\left[(s_2 - \mathbf{w}_s^\top \mathbf{e}(s_1))^2\right]
\]
The gradient becomes:
\begin{equation}
\mathbf{g}_s = -2 \mathbb{E}\left[\mathbf{e}(s_1)\mathbf{e}(s_1)^\top\right](\mathbf{w}^\circ - \mathbf{w}_s)
\label{eq:grad_s}
\end{equation}

For model \( \mathcal{M}_l \), the gradient for \( \mathbf{w}_1 \) is:
\begin{equation}
\mathbf{g}_1 = -2 \mathbb{E}\left[\mathbf{e}(s_1)\mathbf{e}(s_1)^\top\right](\mathbf{w}^\circ - \mathbf{w}_1) + 2 \mathbb{E}\left[\mathbf{w}_2^\top \mathbf{e}(s_x)\mathbf{e}(s_1)\right]
\label{eq:grad_l}
\end{equation}

\noindent Step 2: Compare Gradient Components

The inner product for \( \mathcal{M}_l \) contains two terms:
\begin{equation}
\begin{aligned}
    \mathbb{E}\left[\langle \mathbf{g}_1, \mathbf{w}^\circ - \mathbf{w}_1 \rangle \right] 
    &= \underbrace{-2 (\mathbf{w}^\circ - \mathbf{w}_1)^\top \mathbb{E}\left[\mathbf{e}(s_1)\mathbf{e}(s_1)^\top\right] (\mathbf{w}^\circ - \mathbf{w}_1)}_{\text{Matches } \mathcal{M}_s} \\
    &\quad + \underbrace{2 \mathbb{E}\left[\mathbf{w}_2^\top \mathbf{e}(s_x)\mathbf{e}(s_1)^\top (\mathbf{w}^\circ - \mathbf{w}_1)\right]}_{\text{Additional term}}
\end{aligned}
\end{equation}

\noindent Step 3: Effect of Irrelevant Features

Since \( s_x \) is irrelevant (\( s_2 \perp s_x \)):
\begin{itemize}
\item Population truth: \( \mathbf{w}_2 = \mathbf{0} \)
\item Finite data allows \( \mathbf{w}_2^\epsilon \) to fit noise \( \epsilon \)
\item Induces spurious correlation: \( \mathbb{E}\left[\mathbf{e}(s_x)\mathbf{e}(s_1)^\top\right] \neq \mathbf{0} \)
\end{itemize}

This makes the additional term:
\[
2 \mathbb{E}\left[\mathbf{w}_2^\top \mathbf{e}(s_x)\mathbf{e}(s_1)^\top (\mathbf{w}^\circ - \mathbf{w}_1)\right] \neq 0
\]
which \textit{reduces} the magnitude of the gradient component.

\noindent Step 4: Convergence Comparison

For \( \mathcal{M}_s \):
\[
\mathbb{E}\left[\langle \mathbf{g}_s, \mathbf{w}^\circ - \mathbf{w}_s \rangle \right] = -2 (\mathbf{w}^\circ - \mathbf{w}_s)^\top \mathbb{E}\left[\mathbf{e}(s_1)\mathbf{e}(s_1)^\top\right] (\mathbf{w}^\circ - \mathbf{w}_s)
\]

For \( \mathcal{M}_l \), the additional term in Equation~\ref{eq:grad_l} creates:
\[
\mathbb{E}\left[ \left\langle \mathbf{g}_1, \mathbf{w}^\circ - \mathbf{w}_1 \right\rangle \right] < \mathbb{E}\left[ \left\langle \mathbf{g}_s, \mathbf{w}^\circ - \mathbf{w}_s \right\rangle \right]
\]

\noindent \textbf{Conclusion} \\
The irrelevant features in \( \mathcal{M}_l \) reduce the gradient component in the direction of \( \mathbf{w}^\circ \), leading to slower convergence compared to \( \mathcal{M}_s \).
\end{proof}

\section{CoT simulates the target solution}

\begin{definition}[Chain of Thought]
  \vspace{0em}
    \small\begin{align*}
    \vspace{-0.5em}
        \text{Input:} & \; s_1 \mid \cdots \mid s_N \\[1ex]
        \text{CoT Steps:} & \; \langle\text{sep}\rangle \; s_2 \mid G_2 \mid q_2 \mid L_1 \mid L_2 \\
        & \; \langle\text{sep}\rangle \; \cdots \\
        & \; \langle\text{sep}\rangle \; s_{N} \mid G_N \mid q_N \mid L_{N-1} \mid L_N
    \end{align*}
    \label{def:CoT_appendix}
\end{definition}

\begin{proposition}
\label{thm:CoT_implementation}
For any compound problem satisfying Definition \ref{def:CP}, and for any input length bound $n \in \mathbb{N}$, there exists an autoregressive Transformer with:
\begin{itemize}
\item Constant depth $L$
\item Constant hidden dimension $d$
\item Constant number of attention heads $H$
\end{itemize}
where $L$, $d$, and $H$ are independent of $n$, such that the Transformer correctly generates the Chain-of-Thought solution defined in Definition \ref{def:CoT} for all input sequences of length at most $n$. Furthermore, all parameter values in the Transformer are bounded by $O(\text{poly}(n))$.
\end{proposition}

\subsection{Constructive Proof}
We prove this theorem by constructing a Transformer architecture with 4 blocks, where each block contains multiple attention heads and feed-forward networks (FFNs). The key insight is that we can simulate each step of the Chain-of-Thought solution using a fixed number of attention heads and a fixed embedding dimension.
The attention mechanism is primarily used to select and retrieve relevant elements from the input and previous computations, while the FFNs approximate the required functions $G$, $B$, etc. By maintaining constant depth, width, and number of heads per layer, we ensure the Transformer's architecture remains independent of the input length, while still being able to generate arbitrarily long Chain-of-Thought solutions.
The parameter complexity of $O(\text{poly}(n))$ arises from the need to handle inputs and intermediate computations of length $n$, but importantly, this only affects the parameter values and not the model architecture itself.

\subsection{Embedding Structure}
For position $k$, define the input embedding:
\begin{equation*}
    x^{(0)}_k=(e^{\text{isInput}}_k, e^{\text{isState}}_k, e^{\text{isDependence}}, e^{\text{isL}}, e^{\text{q}}_k, e^{\text{d}}_{k}, e^{\text{L}}_k, e^{\text{sep}}_k, e^{\text{step}}_k, k,1)
\end{equation*}
where:
\begin{itemize}
    \item $e^{\text{isInput}}_k \in \{0,1\}$: Input token indicator
    \item $e^{\text{isState}}_k \in \{0,1\}$: State position indicator
    \item $e^{\text{isDependence}} \in \{0,1\}$: Dependency marker
    \item $e^{\text{isL}} \in \{0,1\}$: Aggregation result indicator
    \item $e^{\text{q}}_k \in \mathbb{R}^{d_q}$: State value embedding
    \item $e^{\text{d}}_{k} \in \mathbb{R}^{d_d}$: Dependency graph embedding
    \item $e^{\text{L}}_k \in \mathbb{R}^{d_L}$: Aggregation value embedding
    \item $e^{\text{sep}}_k \in \{0,1\}$: Step separator indicator
    \item $e^{\text{step}}_k \in \mathbb{N}$: Current step index
    \item $k \in \mathbb{N}$: Position encoding
    \item $1$: Bias term
\end{itemize}

\subsection{Block Constructions}

Block 1: Input Processing and State Identification
Define attention heads $A^{(1)}_1, A^{(1)}_2, A^{(1)}_3$ with parameters:
\begin{align*}
    Q^{(1)}_1 &= W^q_1[e^{\text{isInput}}_k] \\ 
    K^{(1)}_1 &= W^k_1[e^{\text{isInput}}_j]_{j<k} \\
    V^{(1)}_1 &= W^v_1[j]_{j<k}
\end{align*}

The second head tracks state positions:
\begin{align*}
    Q^{(1)}_2 &= W^q_2[e^{\text{isState}}_k] \\
    K^{(1)}_2 &= W^k_2[e^{\text{isState}}_j]_{j<k} \\
    V^{(1)}_2 &= W^v_2[j]_{j<k}
\end{align*}

The third head tracks step indices through separators:
\begin{align*}
    Q^{(1)}_3 &= W^q_3[e^{\text{sep}}_k] \\
    K^{(1)}_3 &= W^k_3[e^{\text{sep}}_j]_{j<k} \\
    V^{(1)}_3 &= W^v_3[\text{count}(e^{\text{sep}}_j)]_{j<k}
\end{align*}

\begin{lemma}
The first block correctly identifies positions through attention scoring:
\begin{enumerate}
    \item For input positions, $A^{(1)}_1$ scoring gives:
    \begin{equation*}
        \text{score}_1(q_k, k_j) = \begin{cases}
        1 & \text{if } e^{\text{isInput}}_j = 1 \\
        0 & \text{otherwise}
        \end{cases}
    \end{equation*}
    Thus $V^{(1)}_1$ returns positions of input tokens

    \item For state positions, $A^{(1)}_2$ scoring gives:
    \begin{equation*}
        \text{score}_2(q_k, k_j) = \begin{cases}
        1 & \text{if } e^{\text{isState}}_j = 1 \\
        0 & \text{otherwise}
        \end{cases}
    \end{equation*} 
    Thus $V^{(1)}_2$ returns positions of states

    \item For step indices, $A^{(1)}_3$ counts separators up to position k:
    \begin{equation*}
        \text{count}(e^{\text{sep}}_j) = \sum_{l \leq j} e^{\text{sep}}_l
    \end{equation*}
    Thus $V^{(1)}_3$ returns the current step index
\end{enumerate}
\end{lemma}

Block 2: Dependency Graph Construction
Define three attention heads $A^{(2)}_1, A^{(2)}_2, A^{(2)}_3$ implementing dependency selection:
\begin{align*}
    A^{(2)}_1&: Q^{(2)}_1 = W^q_2[e^{\text{step}}_k] \\
    &K^{(2)}_1 = W^k_2[e^{\text{input}}_j]_{j<k} \\
    &V^{(2)}_1 = W^v_2[j]_{j<k} \\
    A^{(2)}_2&: Q^{(2)}_2 = W^q_3[e^{\text{step}}_k] \\
    &K^{(2)}_2 = W^k_3[e^{\text{step}}_j]_{j<k} \\
    &V^{(2)}_2 = W^v_3[B(s_1,\ldots,s_{i+1}, i+1)]_{j<k} \\
    A^{(2)}_3&: Q^{(2)}_3 = W^q_4[e^{\text{step}}_k] \\
    &K^{(2)}_3 = W^k_4[j]_{j<k} \\
    &V^{(2)}_3 = W^v_4[e^{\text{q}}_j]_{j<k}
\end{align*}

\begin{lemma}
Block 2 correctly implements $G_{i+1} = \{q_k | k \in B(s_1,\ldots,s_{i+1}, i+1)\}$ through:

1. First attention head $A^{(2)}_1$ gathers input sequence up to current step i+1:
\begin{equation*}
    z^{(2)}_1 = \{s_j | j \leq i+1\}
\end{equation*}

2. Second attention head $A^{(2)}_2$ computes indices from B using gathered inputs:
\begin{equation*}
    z^{(2)}_2 = B(z^{(2)}_1, i+1)
\end{equation*}

3. Third attention head $A^{(2)}_3$ selects states using computed indices:
\begin{equation*}
    z^{(2)}_3 = \{e^{\text{q}}_j | j \in z^{(2)}_2\}
\end{equation*}

Therefore, the composition $z^{(2)}_3(z^{(2)}_2(z^{(2)}_1))$ correctly implements $G_{i+1}$ by:
\begin{enumerate}
    \item Gathering relevant input sequence 
    \item Computing dependency indices using B
    \item Selecting corresponding states
\end{enumerate}

The correctness follows from attention scoring:
\begin{align*}
    \text{score}_1(q_k, k_j) &= \begin{cases}
        1 & \text{if } j \leq i+1 \\
        0 & \text{otherwise}
    \end{cases} \\
    \text{score}_2(q_k, k_j) &= \begin{cases}
        1 & \text{if } j \in B(s_1,\ldots,s_{i+1}, i+1) \\
        0 & \text{otherwise}
    \end{cases} \\
    \text{score}_3(q_k, k_j) &= \begin{cases}
        1 & \text{if } j \in z^{(2)}_2 \\
        0 & \text{otherwise}
    \end{cases}
\end{align*}
\end{lemma}

Block 3: State Transition
Define attention mechanism implementing $F$:
\begin{align*}
    A^{(3)}_1&: Q^{(3)}_1 = W^q_3[e^{\text{isState}}_k] \\
    &K^{(3)}_1 = W^k_3[e^{\text{isDependence}}_j]_{j<k} \\
    &V^{(3)}_1 = W^v_3[e^{\text{q}}_j]_{j<k} \\
    A^{(3)}_2&: Q^{(3)}_2 = W^q_4[e^{\text{isState}}_k] \\
    &K^{(3)}_2 = W^k_4[e^{\text{isInput}}_j]_{j<k} \\
    &V^{(3)}_2 = W^v_4[e^{\text{input}}_j]_{j<k}
\end{align*}

\begin{lemma}
The state transition function $F$ is correctly computed through:
\begin{equation*}
    q_{i+1} = F(G_{i+1}, s_{i+1}) = \text{FFN}(z^{(3)}_1, z^{(3)}_2)
\end{equation*}
where $z^{(3)}_1 = A^{(3)}_1(e^{\text{q}}_j \mid j \in B(s_1,\ldots,s_{i+1}, i+1))$ represents the states selected by $G_{i+1}$ from Block 2, and $z^{(3)}_2 = A^{(3)}_2(s_{i+1})$ represents the current input token.
\end{lemma}

Block 4: Result Aggregation
Define two attention heads $A^{(4)}_1, A^{(4)}_2$ for implementing $H$:
\begin{align*}
    A^{(4)}_1&: Q^{(4)}_1 = W^q_4[e^{\text{isL}}_k] \\
    &K^{(4)}_1 = W^k_4[e^{\text{isL}}_j]_{j<k} \\
    &V^{(4)}_1 = W^v_4[e^{\text{L}}_j]_{j<k} \\
    A^{(4)}_2&: Q^{(4)}_2 = W^q_5[e^{\text{isL}}_k] \\
    &K^{(4)}_2 = W^k_5[e^{\text{isState}}_j]_{j<k} \\
    &V^{(4)}_2 = W^v_5[e^{\text{q}}_j]_{j<k}
\end{align*}

\begin{lemma}
Block 4 correctly implements the aggregation function $H$ through:

1. For $i=1$ (base case):
\begin{equation*}
    \text{score}_1(q_k, k_j) = 0, \quad \text{score}_2(q_k, k_j) = \begin{cases}
        1 & \text{if } e^{\text{isState}}_j = 1 \\
        0 & \text{otherwise}
    \end{cases}
\end{equation*}
Therefore $L_1 = H(\emptyset, q_1) = q_1$ since only $A^{(4)}_2$ activates to select $q_1$

2. For $i>1$:
\begin{equation*}
    \text{score}_1(q_k, k_j) = \begin{cases}
        1 & \text{if } e^{\text{isL}}_j = 1 \text{ and j is the latest L position} \\
        0 & \text{otherwise}
    \end{cases}
\end{equation*}
\begin{equation*}
    \text{score}_2(q_k, k_j) = \begin{cases}
        1 & \text{if } e^{\text{isState}}_j = 1 \text{ and j corresponds to } q_i \\
        0 & \text{otherwise}
    \end{cases}
\end{equation*}

Therefore:
\begin{align*}
    z^{(4)}_1 &= A^{(4)}_1(e^{\text{L}}_k) = L_{i-1} \text{ (previous aggregation result)} \\
    z^{(4)}_2 &= A^{(4)}_2(e^{\text{q}}_k) = q_i \text{ (current state)} \\
    L_i &= \text{FFN}(z^{(4)}_1, z^{(4)}_2) = H(L_{i-1}, q_i)
\end{align*}

The FFN is constructed to implement the specific aggregation operation of $H$ (e.g., max, min, or sum).
\end{lemma}

\begin{proposition}[Block Transitions]
The blocks connect sequentially where:
\begin{enumerate}
    \item Block 1 output provides input positions, state positions and step indices
    \item Block 2 implements dependency function $G$ to gather required states
    \item Block 3 uses gathered dependencies and current input to compute new states via $F$
    \item Block 4 implements $H$ to aggregate states into final result
\end{enumerate}
Each transition preserves information through residual connections.
\end{proposition}

\section{Distribution Shift Analysis for Q-A vs.\ Q-CoT}
\label{app:shift_ana}
Understanding the distribution shift between traditional Q-A and CoT approaches is crucial for analyzing OOD generalization. This section examines how CoT's intermediate reasoning steps influence the alignment between training and evaluation distributions.
We analyze a dynamic state-transition system where problems of different lengths are processed through either direct Q-A or through CoT reasoning \ref{fig:disbution_shift}. Specifically, we compare:
1. \textbf{Q-A training/evaluation} on problem lengths \(\{n_1, n_2\}\) (train) and \(n_3\) (eval).
2. \textbf{Q-CoT training/evaluation} on problem lengths \(\{n_1, n_2\}\) (train) and \(n_3\) (eval).
We assume \(n_{1} < n_{3} < n_{2}\), and let
\[
   \mathcal{D}_{\text{train}}^{\text{Q-A}} 
   \;=\;
   \Bigl\{
   (X^{n_1}, Y^{n_1}),\, 
   (X^{n_2}, Y^{n_2})
   \Bigr\}, 
   \quad \]
\[
   \mathcal{D}_{\text{eval}}^{\text{Q-A}}
   \;=\;
   \Bigl\{(X^{n_3}, Y^{n_3})\Bigr\},
\]
be the respective datasets in the Q-A setup.  Analogously, in the Q-CoT setup, each problem instance is expanded into subproblem–subsolution pairs (the chain-of-thought states).  Denote:
\[
   \mathcal{D}_{\text{train}}^{\text{Q-CoT}} 
   \;=\;
   \Bigl\{
   \bigl(X^{n_1}, \{q_i^{(n_1)}\}, Y^{n_1}\bigr),\, 
   \bigl(X^{n_2}, \{q_i^{(n_2)}\}, Y^{n_2}\bigr)
   \Bigr\},
   \quad \]
\[
   \mathcal{D}_{\text{eval}}^{\text{Q-CoT}}
   \;=\;
   \Bigl\{\bigl(X^{n_3}, \{q_i^{(n_3)}\}, Y^{n_3}\bigr)\Bigr\}.
\]
Here \(\{q_i^{(n)}\}\) denotes the chain-of-thought states or subsolutions for a length-\(n\) problem. 

We let \(P_{\text{train}}^{\text{Q-A}}\), \(P_{\text{eval}}^{\text{Q-A}}\), \(P_{\text{train}}^{\text{Q-CoT}}\), and \(P_{\text{eval}}^{\text{Q-CoT}}\) be the corresponding data or model-induced distributions over Q-A setting or Q-CoT setting.

Our analysis begins with a fundamental result about the structure of CoT sequences in our dynamic system.


\subsection{Prefix-Substructure Theorem in a Dynamic State-Transition System}

A key property of chain-of-thought sequences is that prefix relationships between input sequences carry over to their states, enabling shorter problems CP($n_3$) to embed within longer ones CP($n_2$).

\begin{lemma}[Prefix Substructure]
\label{thm:prefix-substructure}
Let \(S^{short} = (s_1, s_2, \ldots, s_{n_2})\) and let \(S^{long} = (s_1, s_2, \ldots, s_{n_3})\) be its prefix, with \(n_3 < n_2\).  Suppose their chain-of-thought sequences under the same \((G,F)\) are
\[
   Q^{short} \;=\;
   \bigl(q_1^{short}, q_2^{short}, \ldots, q_{n_2}^{short}\bigr)
   \quad\text{and}
\]
\[\quad
   Q^{long} \;=\;
   \bigl(q_1^{long}, q_2^{long}, \ldots, q_{n_3}^{long}\bigr).
\]
Then for all \(1 \le i \le n_3\), we have
\[
   q_i^{long} 
   \;=\; 
   q_i^{short}.
\]
Hence \(\bigl(q_1^{long},\ldots,q_{n_3}^{long}\bigr)\) is exactly the prefix of \(\bigl(q_1^{short},\ldots,q_{n_2}^{short}\bigr)\).
\end{lemma}

This property suggests CoT's intermediate steps provide natural bridges between problems of different lengths, potentially easing distribution shift concerns.
\subsection{Training Size Effects on Distribution Shift Through Prefix Coverage}
Building upon the prefix-substructure property, we now quantify the distribution shift between training and evaluation sets for both Q-A and Q-CoT approaches. This analysis reveals how CoT's intermediate reasoning steps can potentially mitigate distribution shift effects.
\begin{theorem}[KL Divergence Reduction via Chain-of-Thought Coverage]
\label{thm:kl-reduction}
Let $m_1$ and $m_2$ denote the number of training sequences of length $n_1$ and $n_2$, respectively, with total training size $M = m_1 + m_2$. Suppose the evaluation set consists of $m_3$ sequences of length $n_3$, with $n_1 < n_3 < n_2$. Let $k$ be the input vocabulary size.

\textbf{(Coverage Probability Definition)}\\
Let the \emph{coverage probability} $P_{\text{cover}}(M)$ be defined as the probability that an intermediate state $q_{1:n_3}^{(\mathrm{eval})}$, corresponding to a prefix of length $n_3$ from an evaluation sequence, is contained within the support of the training distribution; that is, $q_{1:n_3}^{(\mathrm{eval})}$ exists as the prefix of some chain-of-thought sequence $q_{1:n_2}^{(\mathrm{train})}$ in the training data. Formally,
\[
P_{\text{cover}}(M) = \mathbb{P}\left(
q_{1:n_3}^{(\mathrm{eval})}
\in 
\left\{
q_{1:n_3}^{(\mathrm{train})}\,:\,
q_{1:n_3}^{(\mathrm{train})} \text{ is prefix of length } n_3\text{ from some training } q_{1:n_2}^{(\mathrm{train})}
\right\}
\right).
\]

Then, the following holds:
\begin{enumerate}
    \item \textbf{Prefix Coverage:} The coverage probability for evaluation prefixes of length $n_3$ is
    \[
    P_{\text{cover}}(M) = \frac{m_2}{m_3 k^{n_3}}\ ,
    \]
    where $m_3$ is the size of the evaluation set and $m_2 \leq m_3 k^{n_3}$. 

    \item \textbf{KL Divergence Reduction:} The Kullback-Leibler (KL) divergence between evaluation and training distributions over intermediate reasoning steps (Q-CoT supervision) satisfies
    \[
    D_{\mathrm{KL}}\!\Bigl(
        P_{\mathrm{eval}}^{\mathrm{Q\text{-}CoT}}
        \,\|\,
        P_{\mathrm{train}}^{\mathrm{Q\text{-}CoT}}
    \Bigr)
    \leq
    \left( 1 - P_{\text{cover}}(M) \right)
    D_{\mathrm{KL}}\!\Bigl(
        P_{\mathrm{eval}}^{\mathrm{Q\text{-}A}}
        \,\|\,
        P_{\mathrm{train}}^{\mathrm{Q\text{-}A}}
    \Bigr)\ .
    \]
    Thus, supervision on intermediate steps—i.e., Chain-of-Thought (CoT)—dramatically reduces the distributional gap between training and evaluation compared to direct QA supervision.

    \item \textbf{Complete Coverage Yields No Shift:} In the limit where $m_2 = m_3 k^{n_3}$ (all evaluation prefixes are seen in training),
    \[
    D_{\mathrm{KL}}\!\Bigl(
        P_{\mathrm{eval}}^{\mathrm{Q\text{-}CoT}}
        \,\|\,
        P_{\mathrm{train}}^{\mathrm{Q\text{-}CoT}}
    \Bigr) = 0\ .
    \]
\end{enumerate}
\end{theorem}

\begin{remark}[Interpretation and Practical Role of CoT]
This result explicitly quantifies how CoT mitigates the KL divergence (distribution shift) between training and evaluation. The coverage probability $P_{\text{cover}}(M)$ measures how likely it is for a generated intermediate (reasoning-step) prefix in evaluation to appear in the CoT-augmented training data. As $P_{\text{cover}}(M)$ increases with more or better-constructed CoT training data, the effective distribution shift shrinks, ensuring robust generalization—even where direct QA fails. In practice, incomplete or lower-granularity CoT chains reduce $P_{\text{cover}}(M)$, revealing a concrete trade-off between annotation effort and OOD generalization.
\end{remark}

\subsection{Proof for Theorem \ref{thm:kl-reduction}}

\begin{proof}
We prove the KL divergence bound by decomposing the distributions over covered and uncovered prefixes. Let $P_{\text{train}}^{\text{Q-CoT}}$ and $P_{\text{eval}}^{\text{Q-CoT}}$ denote the training and evaluation distributions under Q-CoT, respectively. 

\vspace{0.5em}

\noindent Step 1: Event Space Partitioning

Define two disjoint events for any evaluation sample $x = (X^{n_3}, \{q_i^{(n_3)}\}, Y^{n_3})$:
\begin{itemize}
    \item $\mathcal{E}_{\text{cover}}$: The prefix $\{q_i^{(n_3)}\}_{i=1}^{n_3}$ exists in some length-$n_2$ training sample.
    \item $\mathcal{E}_{\text{uncover}}$: The prefix $\{q_i^{(n_3)}\}_{i=1}^{n_3}$ is absent from all training samples.
\end{itemize}
By Lemma \ref{thm:prefix-substructure}, $\mathcal{E}_{\text{cover}}$ occurs when the evaluation prefix matches at least one length-$n_2$ training sequence's prefix. The probabilities satisfy:
\[
P_{\text{cover}} = \mathbb{P}(\mathcal{E}_{\text{cover}}), \quad 1 - P_{\text{cover}} = \mathbb{P}(\mathcal{E}_{\text{uncover}}).
\]
where $P_{\text{cover}}$ is calculated as:
\[
P_{\text{cover}} = \frac{m_2}{m_3 k^{n_3}}
\]
\textit{Derivation}: Each length-$n_2$ training sample contains a unique prefix of length $n_3$ (Lemma \ref{thm:prefix-substructure}). With $m_2$ samples, we can cover $m_2$ distinct prefixes. The total number of possible prefixes is $m_3 k^{n_3}$ ($m_3$ evaluation problems, each with $k^{n_3}$ possible prefixes). Thus, the coverage probability follows the ratio.

\vspace{0.5em}

\noindent Step 2: Distributional Decomposition

Using the law of total probability, we express:
\[
P_{\text{eval}}^{\text{Q-CoT}} = P_{\text{cover}} \cdot P_{\text{eval}|\mathcal{E}_{\text{cover}}} + (1-P_{\text{cover}}) \cdot P_{\text{eval}|\mathcal{E}_{\text{uncover}}}
\]
\[
P_{\text{train}}^{\text{Q-CoT}} = P_{\text{cover}} \cdot P_{\text{train}|\mathcal{E}_{\text{cover}}} + (1-P_{\text{cover}}) \cdot P_{\text{train}|\mathcal{E}_{\text{uncover}}}
\]
where:
\begin{itemize}
    \item $P_{\text{eval}|\mathcal{E}_{\text{cover}}}$: Evaluation distribution restricted to covered prefixes
    \item $P_{\text{train}|\mathcal{E}_{\text{cover}}}$: Training distribution restricted to covered prefixes
    \item $P_{\text{eval}|\mathcal{E}_{\text{uncover}}}$: Evaluation distribution for uncovered prefixes
    \item $P_{\text{train}|\mathcal{E}_{\text{uncover}}}$: Training distribution for uncovered prefixes
\end{itemize}

\vspace{0.5em}

\noindent Step 3: KL Divergence Expansion with Total Expectation

From the KL divergence definition:
\[
D_{\mathrm{KL}}\left(P_{\text{eval}}^{\text{Q-CoT}} \,\big\|\, P_{\text{train}}^{\text{Q-CoT}}\right) = \mathbb{E}_{x \sim P_{\text{eval}}^{\text{Q-CoT}}} \left[ \log \frac{P_{\text{eval}}^{\text{Q-CoT}}(x)}{P_{\text{train}}^{\text{Q-CoT}}(x)} \right]
\]
Apply the law of total expectation by conditioning on $\mathcal{E}_{\text{cover}}$ and $\mathcal{E}_{\text{uncover}}$:
\[
= \mathbb{P}(\mathcal{E}_{\text{cover}}) \cdot \mathbb{E}_{x|\mathcal{E}_{\text{cover}}} \left[ \log \frac{P_{\text{eval}}^{\text{Q-CoT}}(x|\mathcal{E}_{\text{cover}})}{P_{\text{train}}^{\text{Q-CoT}}(x|\mathcal{E}_{\text{cover}})} \right] + \mathbb{P}(\mathcal{E}_{\text{uncover}}) \cdot \mathbb{E}_{x|\mathcal{E}_{\text{uncover}}} \left[ \log \frac{P_{\text{eval}}^{\text{Q-CoT}}(x|\mathcal{E}_{\text{uncover}})}{P_{\text{train}}^{\text{Q-CoT}}(x|\mathcal{E}_{\text{uncover}})} \right]
\]

\vspace{0.5em}

\noindent Step 4: Handling Covered Cases

Under $\mathcal{E}_{\text{cover}}$, Lemma \ref{thm:prefix-substructure} guarantees that the CoT states $\{q_i^{(n_3)}\}$ in evaluation samples exactly match those in training samples. This implies:
\[
P_{\text{eval}|\mathcal{E}_{\text{cover}}}(x) = P_{\text{train}|\mathcal{E}_{\text{cover}}}(x), \quad \forall x \in \mathcal{E}_{\text{cover}}
\]
Therefore:
\[
\mathbb{E}_{x|\mathcal{E}_{\text{cover}}} \left[ \log \frac{P_{\text{eval}|\mathcal{E}_{\text{cover}}}}{P_{\text{train}|\mathcal{E}_{\text{cover}}}} \right] = \mathbb{E}_{x|\mathcal{E}_{\text{cover}}} [\log 1] = 0
\]

\vspace{0.5em}

\noindent Step 5: Uncovered Cases Reduce to Q-A

For $x = (X^{n_3}, \{q_i^{(n_3)}\}, Y^{n_3}) \in \mathcal{E}_{\text{uncover}}$, the absence of matching prefixes in training data implies the model cannot leverage CoT states $\{q_i^{(n_3)}\}$ during inference. We formally analyze this degradation:



Under Q-CoT, the generation process factors as:
\[
P^{\text{Q-CoT}}(Y|X) = \sum_{\{q_i\}} P(Y|X, \{q_i\}) P(\{q_i\}|X)
\]
where:
\begin{itemize}
    \item $P(\{q_i\}|X)$: Probability of generating CoT states $\{q_i\}$ given input $X$
    \item $P(Y|X, \{q_i\})$: Probability of answer $Y$ given $X$ and CoT states
\end{itemize}



When $\{q_i^{(n_3)}\}$ is uncovered ($\mathcal{E}_{\text{uncover}}$), the model lacks training data to estimate either:
\begin{itemize}
    \item The CoT state distribution $P(\{q_i\}|X)$
    \item The answer likelihood $P(Y|X, \{q_i\})$ 
\end{itemize}

Thus, the model \textit{cannot} utilize the CoT decomposition and must marginalize over all possible $\{q_i\}$:
\[
P^{\text{Q-CoT}}(Y|X) = \mathbb{E}_{\{q_i\} \sim P(\{q_i\}|X)} \left[ P(Y|X, \{q_i\}) \right]
\]



Without CoT supervision on $\{q_i^{(n_3)}\}$, two condition assumes:
\begin{enumerate}
    \item \textbf{Untrained CoT States}: If $\{q_i^{(n_3)}\}$ never appears in training, $P(\{q_i\}|X)$ becomes a \textit{uniform prior} over possible CoT sequences (by maximum entropy principle).
    
    \item \textbf{Uninformative Likelihood}: The answer likelihood $P(Y|X, \{q_i\})$ reduces to $P^{\text{Q-A}}(Y|X)$ because the model cannot associate $\{q_i\}$ with $Y$ without training signals.
\end{enumerate}

Thus:
\[
P^{\text{Q-CoT}}(Y|X) = \sum_{\{q_i\}} \underbrace{P^{\text{Q-A}}(Y|X)}_{\text{Uninformative}} \cdot \underbrace{\frac{1}{k^{n_{3}}}}_{\text{Uniform } P(\{q_i\}|X)} = P^{\text{Q-A}}(Y|X)
\]

with expansion of KL divergence of Q-A
\[
D_{\mathrm{KL}}\left(P_{\text{eval}}^{\text{Q-A}} \,\big\|\, P_{\text{train}}^{\text{Q-A}}\right) = \mathbb{E}_{x \sim P_{\text{eval}}^{\text{Q-A}}} \left[ \log \frac{P_{\text{eval}}^{\text{Q-A}}(x)}{P_{\text{train}}^{\text{Q-A}}(x)} \right]
\] 
\[
= \mathbb{P}(\mathcal{E}_{\text{cover}}) \cdot \mathbb{E}_{x|\mathcal{E}_{\text{cover}}} \left[ \log \frac{P_{\text{eval}}^{\text{Q-A}}(x|\mathcal{E}_{\text{cover}})}{P_{\text{train}}^{\text{Q-A}}(x|\mathcal{E}_{\text{cover}})} \right] + \mathbb{P}(\mathcal{E}_{\text{uncover}}) \cdot \mathbb{E}_{x|\mathcal{E}_{\text{uncover}}} \left[ \log \frac{P_{\text{eval}}^{\text{Q-A}}(x|\mathcal{E}_{\text{uncover}})}{P_{\text{train}}^{\text{Q-A}}(x|\mathcal{E}_{\text{uncover}})} \right]
\]
Notice that
\[
\mathbb{E}_{x|\mathcal{E}_{\text{cover}}} \left[ \log \frac{P_{\text{eval}}^{\text{Q-A}}(x|\mathcal{E}_{\text{cover}})}{P_{\text{train}}^{\text{Q-A}}(x|\mathcal{E}_{\text{cover}})} \right]
\leq
\mathbb{E}_{x|\mathcal{E}_{\text{uncover}}} \left[ \log \frac{P_{\text{eval}}^{\text{Q-A}}(x|\mathcal{E}_{\text{uncover}})}{P_{\text{train}}^{\text{Q-A}}(x|\mathcal{E}_{\text{uncover}})} \right]
\]
since covered prefix will decrease the KL divergence via probability decomposition
\[
D_{\mathrm{KL}}\left(P_{\text{eval}}^{\text{Q-A}} \,\big\|\, P_{\text{train}}^{\text{Q-A}}\right) \geq 
\mathbb{E}_{x|\mathcal{E}_{\text{uncover}}} \left[ \log \frac{P_{\text{eval}}^{\text{Q-A}}(x|\mathcal{E}_{\text{uncover}})}{P_{\text{train}}^{\text{Q-A}}(x|\mathcal{E}_{\text{uncover}})} \right] = D_{\mathrm{KL}}\left(P_{\text{eval}|\mathcal{E}_{\text{uncover}}}^{\text{Q-A}} \,\big\|\, P_{\text{train}|\mathcal{E}_{\text{uncover}}}^{\text{Q-A}}\right)
\]

Therefore, for $x \in \mathcal{E}_{\text{uncover}}$:
\[
D_{\mathrm{KL}}\left(P_{\text{eval}|\mathcal{E}_{\text{uncover}}} \,\big\|\, P_{\text{train}|\mathcal{E}_{\text{uncover}}}\right) = 
D_{\mathrm{KL}}\left(P_{\text{eval}|\mathcal{E}_{\text{uncover}}}^{\text{Q-A}} \,\big\|\, P_{\text{train}|\mathcal{E}_{\text{uncover}}}^{\text{Q-A}}\right) \leq 
D_{\mathrm{KL}}\left(P_{\text{eval}}^{\text{Q-A}} \,\big\|\, P_{\text{train}}^{\text{Q-A}}\right) = \mathrm{KL}_{\text{base}}
\]

\vspace{0.5em}

\noindent Step 6: Final Inequality

Combining all terms:
\[
D_{\mathrm{KL}}\left(P_{\text{eval}}^{\text{Q-CoT}} \,\big\|\, P_{\text{train}}^{\text{Q-CoT}}\right) = \underbrace{P_{\text{cover}} \cdot 0}_{\text{Covered term}} + \underbrace{(1-P_{\text{cover}}) \cdot \mathrm{KL}_{\text{base}}}_{\text{Uncovered term}}
\]
Hence:
\[
D_{\mathrm{KL}}\left(P_{\text{eval}}^{\text{Q-CoT}} \,\big\|\, P_{\text{train}}^{\text{Q-CoT}}\right) \leq (1 - P_{\text{cover}}) \cdot \mathrm{KL}_{\text{base}}
\]
The equality holds when $P_{\text{cover}} \in [0,1]$. When $m_2 = m_3 k^{n_3}$, we have $P_{\text{cover}} = 1$, making the KL divergence zero.
\end{proof}

\section{Quantitation Analysis Drop of CoT}
    \begin{theorem}[CoT Accuracy Degradation]
    Let $s_{\text{input}}$ be the input text, $s_{\text{ans}}$ be the unique correct answer, and $s_1, \dots, s_k$ be the \textit{exact required sequence} of perfect Chain-of-Thought (CoT) tokens where:
    \begin{enumerate}
        \item \textbf{Completeness}: $P(s_{\text{ans}} \mid s_1, \dots, s_k, s_{\text{input}}) = 1$
        \item \textbf{Uniqueness}: No other token sequence produces $s_{\text{ans}}$
        \item \textbf{Conditional Independence}: $P(s_1, \dots, s_k \mid s_{\text{input}}) = \prod_{i=1}^k P(s_i \mid s_{\text{input}})$
        \item \textbf{Training Deficiency}: For any CoT token $s_j$ excluded during training, $P(s_j \mid s_{\text{input}})$ drops from 1 to $1 - \epsilon$
    \end{enumerate}
    When $l < k$ CoT tokens are lost/mishandled during inference, the final answer accuracy satisfies:
    \[
    P(s_{\text{ans}} \mid s_{\text{input}}) = (1 - \epsilon)^l
    \]
    \label{thm:drop_CoT}
\end{theorem}

\begin{proof}
By the uniqueness condition, only the full sequence $s_1, \dots, s_k$ guarantees $s_{\text{ans}}$. Let $\mathcal{L}$ be the set of $l$ compromised tokens. The probability of maintaining correctness is:

\[
P(s_{\text{ans}} \mid s_{\text{input}}) = \underbrace{\prod_{j \in \mathcal{L}} P(s_j \mid s_{\text{input}})}_{\text{Lost tokens}} \cdot \underbrace{\prod_{i \notin \mathcal{L}} P(s_i \mid s_{\text{input}})}_{\text{Preserved tokens}}
\]

For preserved tokens ($i \notin \mathcal{L}$), full training ensures $P(s_i \mid s_{\text{input}}) = 1$. For lost tokens ($j \in \mathcal{L}$), training deficiency gives $P(s_j \mid s_{\text{input}}) = 1 - \epsilon$. Thus:

\[
P(s_{\text{ans}} \mid s_{\text{input}}) = (1 - \epsilon)^l \cdot 1^{k-l} = (1 - \epsilon)^l
\]

This equality holds because any deviation from the exact CoT sequence (due to lost tokens) eliminates the chance of correctness by the uniqueness condition.
\end{proof}
\subsection{Experiments}
For the Longest Increasing Subsequence and Multi-Step Path Counting tasks, we implemented a 6-layer transformer architecture trained from scratch, featuring an embedding size of 256/512 and 16 attention heads. The model training utilized the following hyperparameters: maximum sequence length (\texttt{--maxlen}) 524, maximum data samples (\texttt{--maxdata}) 524, vocabulary size (\texttt{--vocab}) 59, number range (\texttt{--num\_range}) 50, weight decay 0.05, learning rate $1 \times 10^{-3}$, dropout 0.1, batch size 256, 1 training epoch, warmup ratio 0.1, model dimension (\texttt{--dmodel}) 256, number of heads 16, number of layers 6, with chain-of-thought supervision (\texttt{--chain}), rotary position embedding (\texttt{--rpe}), and supervised fine-tuning (\texttt{--sft}). Model training was distributed on 4 GPUs using \texttt{torchrun --nproc\_per\_node=4}.

For the Equation Restoration tasks, we adopted a different approach using the Phi-3.5-mini-instruct model as the backbone for task-specific fine-tuning. The restoring equation experiments used the following additional hyperparameters: maximum equation length (\texttt{max\_len}) 300, maximum normalization factor (\texttt{max\_norm}) 1, maximum training samples (\texttt{max\_samples}) 200{,}000, and micro train batch size (\texttt{micro\_train\_batch\_size}) 32.
\subsubsection{LIS}
\label{app:LIS}
Chain of thought is like the following: 
\[
\begin{aligned}
    &48 \quad 49 \quad 26 \quad 47 <sep> \\
    &48 | <empty> = 48 \quad 1 : 1 \rightarrow 1 <sep>\\
    &49 | 48 \quad 1 = 49 \quad 2 : 1 \rightarrow 2 <sep> \\
    &26 | <empty> = 26 \quad 1 : 2 \rightarrow 2 <sep>\\
    &47 | 26 \quad 1 = 47 \quad 2 : 2 \rightarrow 2
    \end{aligned}
\]

\subsubsection{MPC}
\label{app:MPC}
Chain of thought is like following:
\[
\begin{aligned}
 &0 \quad 1 \quad 1 \quad 0 \quad 0 \quad 1 \quad 1 0 , 8 <sep> \\
 &1 , 0 , 1 \rightarrow 0 <sep> \\
 &2 , 1 , 1 \quad 0 \rightarrow 1 <sep> \\
 &3 , 1 , 1 \quad 0 \quad 1 \rightarrow 2 <sep> \\
 &4 , 0 , 0 \quad 1 \quad 2 \rightarrow 0 <sep> \\
 &5 , 0 , 1 \quad 2 \quad 0 \rightarrow 0 <sep> \\
 &6 , 1 , 2 \quad 0 \quad 0 \rightarrow 2 <sep> \\
 &7 , 1 , 0 \quad 0 \quad 2 \rightarrow 2 <sep> \\
 &8 , 0 , 0 \quad 2 \quad 2 \rightarrow 0
\end{aligned}
\]
\subsubsection{Equation Restoration and Variable Computation}
\label{app:ERVC}
Input is:
Data:
$data_1: Condor = 6, Cheetah = 1.$ \\
$data_2: Condor = 12, Cheetah = 3.$ \\
Question:
Assume all relations between variables are linear combinations. If the number of Cheetah equals 5, then what is the number of Condor?

Question:
Assume all relations between variables are linear combinations. If the number of Leopard equals 5, the number of Rhino equals 3, the number of Koala equals 6, then what is the number of Black\_Bear?

\textbf{Solution\:}

\textbf{Defining Variables} \\
\textit{Known Variables:} \\
Cheetah as \( c_1 = 5 \) \\

\textit{Unknown Variables:} \\
Target Variable: Condor as \( c_2 \) \\

\textbf{Restoring Relations} \\
\textit{List all variable names in each data point:} \([c_2, c_1], [c_2, c_1]\) \\
\textit{Deduplicate them:} \([c_2, c_1]\) \\
There is 1 distinct group, implying 1 distinct linear relationship to be determined. \\
\textit{Examining each relationship:} \\

\textbf{Relation 1:} \\
Exploring relation for \( c_2 \): \\
There are 2 variables in the data beginning with \( c_2 \): Hence, 2 coefficients are required, and at least 2 data points are needed. \\

Let the coefficients on the right side of the equation be \( K_1 \) and \( K_2 \). \\
\textit{Recap variables:} \(['c_2', 'c_1']\) \\
\textit{Define the equation of relation 1:} \\
\( c_2 = K_1 \cdot c_1 + K_2 \) \\

Using data points \( \text{data}_1 \) and \( \text{data}_2 \): \\
\( \text{data}_1: c_2 = 6, c_1 = 1 \) \\
Equation 1: \( 6 = K_1 \cdot 1 + K_2 \) \\
\( \text{data}_2: c_2 = 12, c_1 = 3 \) \\
Equation 2: \( 12 = K_1 \cdot 3 + K_2 \) \\

\textbf{Solve the system of equations using Gaussian Elimination:} \\
\textit{Initialize:} \\
Equation 1: \( 1 \cdot K_1 + 1 \cdot K_2 = 6 \) \\
Equation 2: \( 3 \cdot K_1 + 1 \cdot K_2 = 12 \) \\

Swap Equation 1 with Equation 2: \\
Equation 1: \( 3 \cdot K_1 + 1 \cdot K_2 = 12 \) \\
Equation 2: \( 1 \cdot K_1 + 1 \cdot K_2 = 6 \) \\

Multiply Equation 1 by 1 and subtract 3 times Equation 2: \\
\((\text{Equation 1}) \cdot 1: 3 \cdot K_1 + 1 \cdot K_2 = 12 \) \\
\((\text{Equation 2}) \cdot 3: 3 \cdot K_1 + 3 \cdot K_2 = 18 \) \\
New Equation 2: \( -2 \cdot K_2 = -6 \) \\

\textit{Recap updated equations:} \\
Equation 1: \( 3 \cdot K_1 + 1 \cdot K_2 = 12 \) \\
Equation 2: \( -2 \cdot K_2 = -6 \) \\

\textbf{Solve for \( K_2 \):} \\
\( -2 \cdot K_2 = -6 \) \\
\( K_2 = \frac{-6}{-2} = 3 \) \\

\textbf{Solve for \( K_1 \):} \\
\( 3 \cdot K_1 = 12 - 1 \cdot K_2 \) \\
\( 3 \cdot K_1 = 12 - 3 = 9 \) \\
\( K_1 = \frac{9}{3} = 3 \) \\

\textit{Recap the equation:} \\
\( c_2 = K_1 \cdot c_1 + K_2 \) \\
Estimated coefficients: \( K_1 = 3, K_2 = 3 \) \\
Final equation: \( c_2 = 3 \cdot c_1 + 3 \) \\

\textbf{Calculation with Restored Relations:} \\
Using the equation \( c_2 = 3 \cdot c_1 + 3 \): \\
\textit{Known variables:} \( c_1 = 5 \) \\
\( c_2 = 3 \cdot 5 + 3 = 15 + 3 = 18 \) \\

\textbf{Recap Target Variable:} \\
Condor (\( c_2 \)) = 18 \\

\textbf{Conclusion:} The number of Condor equals 18.

\subsection{Out-of-distribution Comparison Across Input length}
The comparison \ref{fig:ood_detail} reveals the critical role of Chain-of-Thought prompting in improving models' OOD generalization. Both MPC (a) and LIS (b) demonstrate substantially higher accuracy when equipped with 100\% COT (blue lines) compared to without COT. This performance gap is particularly pronounced in out-of-domain regions, where models without COT show severe degradation (dropping below 0.2 accuracy). The consistent superior performance of COT-enabled models, especially in maintaining accuracy above 0.8 across different sequence lengths, underscores how COT prompting serves as a crucial mechanism for enhancing models' ability to generalize beyond their training distribution.
\begin{figure}[]
\centering
\subfigure[]{
    \includegraphics[width=0.85\linewidth]{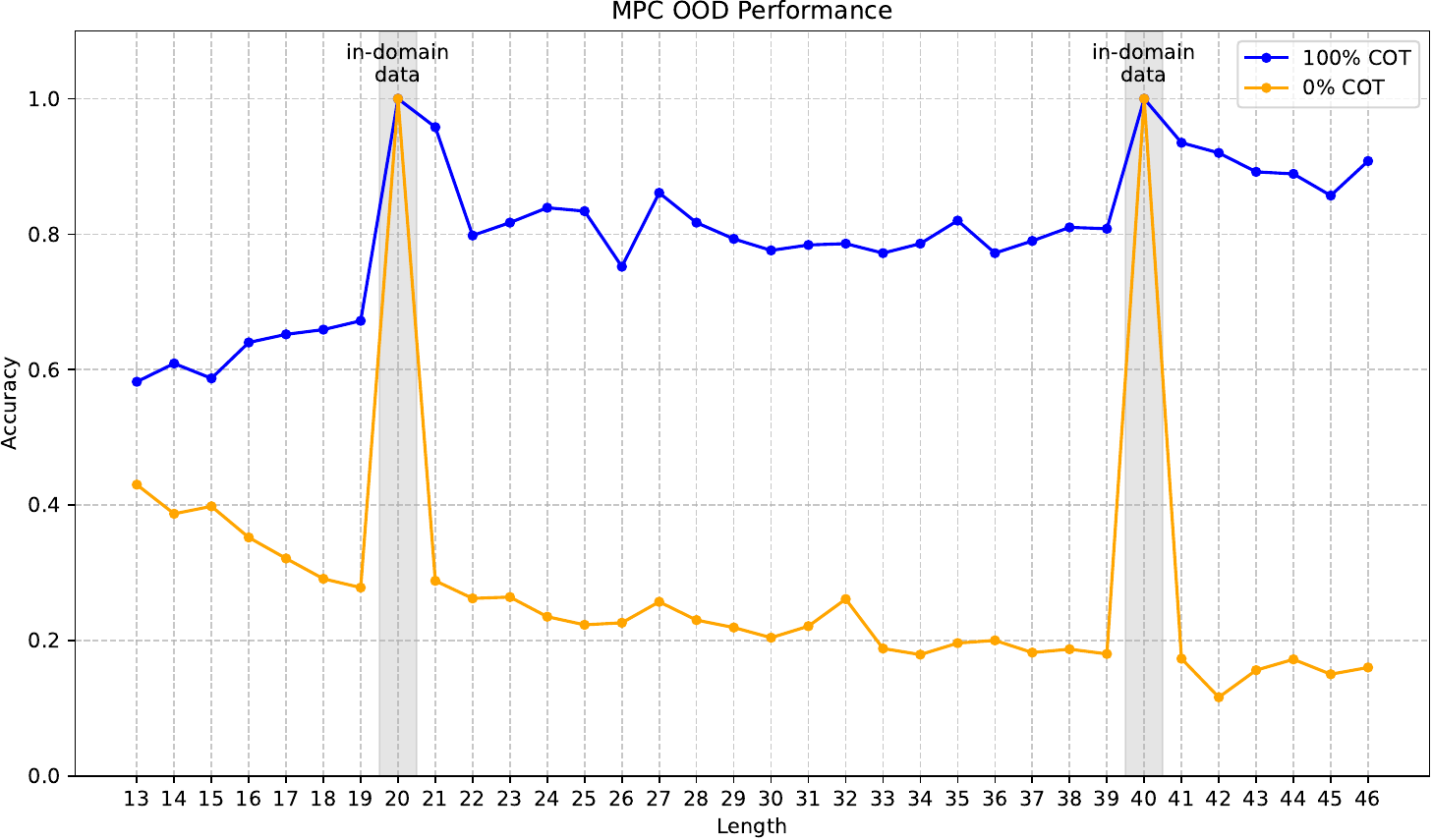}
}\hfill
\subfigure[]{
    \includegraphics[width=0.85\linewidth]{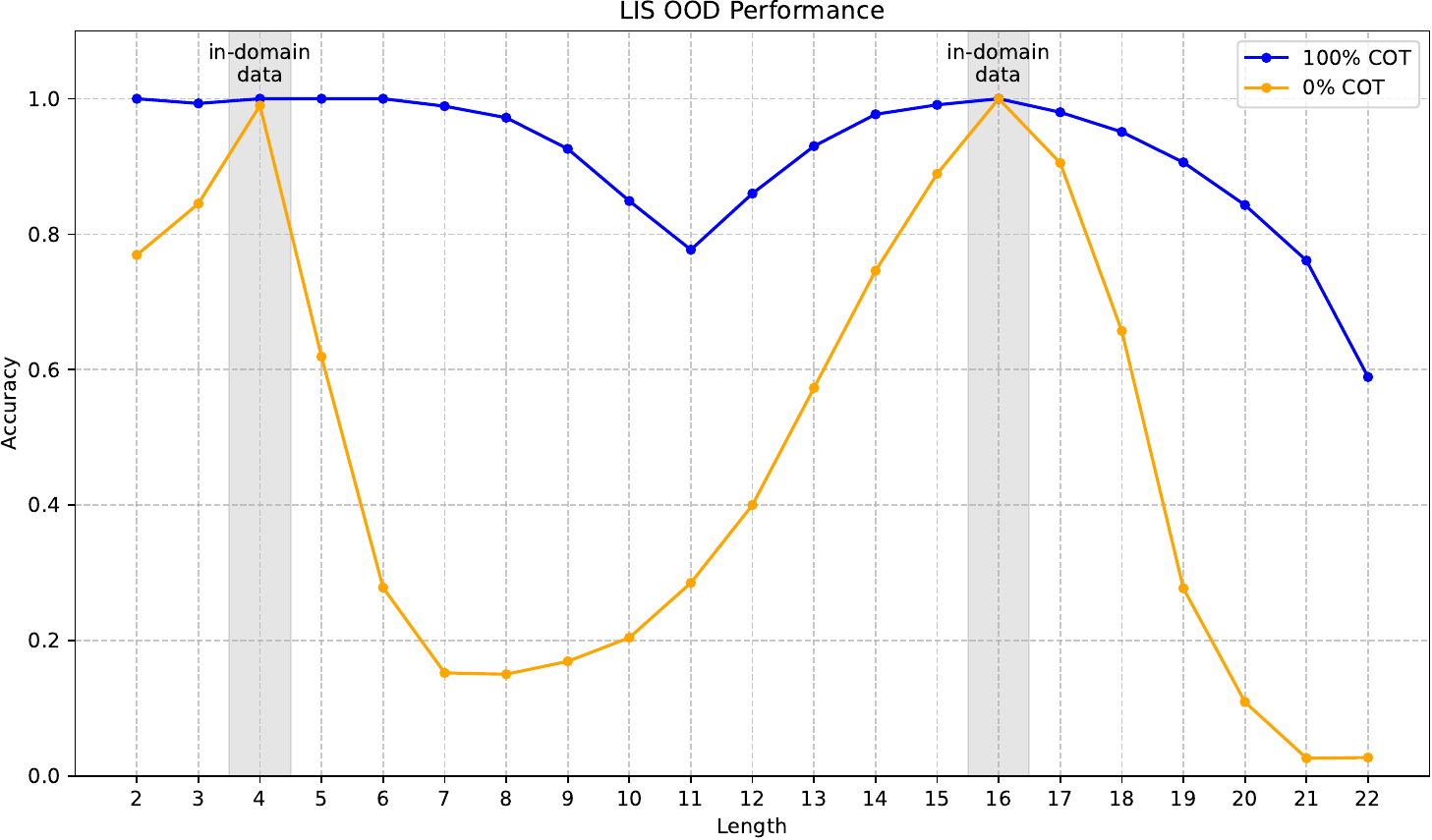}}
\caption{Comparison of Out-Of-Distribution (OOD) performance between MPC and LIS models under different Chain-of-Thought (COT) conditions across varying sequence lengths.}
\label{fig:ood_detail}
\end{figure}
\section*{Impact Statement}
This paper offers a novel perspective, demonstrating the indispensable role of CoT in enhancing the generalization capabilities of LMs. Through theoretical analysis and comprehensive empirical experimentation, we establish CoT as a critical enabler of robust out-of-distribution generalization. Crucially, this work provides valuable guidance for the development of effective data curation strategies, specifically for collecting data that maximizes the benefits of CoT training. This guidance is directly applicable to the industrial deployment of LMs and the fine-tuning of large models for novel tasks, offering a pathway to improve the generalization and real-world utility of these models through informed data acquisition methodologies.
\section*{Use of Large Language Models}
Large Language Models (LLMs) were used exclusively for language polishing, including grammar refinement and improving readability. All research ideas, methodological contributions, experimental design, analysis, and writing of technical content were conceived and carried out solely by the authors. The LLMs did not generate or influence any scientific claims, results, or interpretations. The authors take full responsibility for the content of this paper.


\end{document}